\documentclass{article}

\usepackage[preprint]{neurips_2025}

\usepackage[utf8]{inputenc} 
\usepackage[T1]{fontenc}    
\usepackage{hyperref}       
\usepackage{url}            
\usepackage{booktabs}       
\usepackage{amsfonts}       
\usepackage{nicefrac}       
\usepackage{microtype}      
\usepackage{xcolor}         

\usepackage{graphicx} 

\usepackage{url}            
\usepackage{booktabs}       
\usepackage{amsfonts}       
\usepackage{nicefrac}       
\usepackage{microtype}      
\usepackage{xcolor}         
\usepackage{bm}
\usepackage{amsmath}
\usepackage{amssymb}

\usepackage{multirow}
\usepackage{colortbl}
\definecolor{Gray}{gray}{0.85}
\usepackage{enumitem}

\usepackage{microtype}
\usepackage{graphicx}
\usepackage{subfig}
\usepackage{booktabs} 
\usepackage{verbatim}
\usepackage{amsmath}
\usepackage{mathtools}
\usepackage{dsfont}
\usepackage{amsthm}
\usepackage{lipsum}
\usepackage{mathrsfs}
\usepackage{algorithm}
\usepackage{algorithmic}

\usepackage{amsthm}
\usepackage{comment}
\usepackage{mathtools}

\usepackage{natbib}
\newcommand{\comm}[1]{}
\newcommand{\blue}[1]{\textcolor{blue}{#1}}
\newcommand{\red}[1]{\textcolor{red}{#1}}

\newcommand{\hatZ}{\widehat{\mathbf{Z}}}
\newcommand{\tildeZ}{\widetilde{\mathbf{Z}}}

\newcommand{\DP}{\textsc{DP}}

\newtheorem{theorem}{Theorem}[section]

\newtheorem{lemma}[theorem]{Lemma}

\theoremstyle{definition}
\newtheorem{definition}[theorem]{Definition}
\newtheorem{assumption}[theorem]{Assumption}
\newtheorem{claim}[theorem]{Claim}

\theoremstyle{remark}

\usepackage[utf8]{inputenc} 
\usepackage[T1]{fontenc} 
\usepackage[normalem]{ulem}

\usepackage[textsize=tiny]{todonotes}

\newcommand{\N}{\mathbb{N}}

\newcommand{\R}{\mathbb{R}}

\newcommand{\cB}{\mathcal{B}}
\newcommand{\cC}{\mathcal{C}}
\newcommand{\cD}{\mathcal{D}}

\newcommand{\cH}{\mathcal{H}}

\newcommand{\cM}{\mathcal{M}}
\newcommand{\cN}{\mathcal{N}}
\newcommand{\cO}{\mathcal{O}}
\newcommand{\cP}{\mathcal{P}}
\newcommand{\cQ}{\mathcal{Q}}
\newcommand{\cR}{\mathcal{R}}
\newcommand{\cS}{\mathcal{S}}

\newcommand{\cW}{\mathcal{W}}
\newcommand{\cX}{\mathcal{X}}

\newcommand{\bfk}{\mathbf{k}}

\newcommand{\bfy}{\mathbf{y}}

\newcommand{\bfA}{\mathbf{A}}
\newcommand{\bfB}{\mathbf{B}}
\newcommand{\bfC}{\mathbf{C}}

\newcommand{\bfI}{\mathbf{I}}

\newcommand{\bfK}{\mathbf{K}}

\newcommand{\bfR}{\mathbf{R}}

\newcommand{\bfV}{\mathbf{V}}
\newcommand{\bfW}{\mathbf{W}}

\newcommand{\bfY}{\mathbf{Y}}
\newcommand{\bfZ}{\mathbf{Z}}

\newcommand{\sA}{\mathscr{A}}

\newcommand{\ip}[2] {\langle #1, #2 \rangle }

\title{Differential Privacy in Kernelized Contextual Bandits via Random Projections}

\author{	
	 \textnormal{Nikola Pavlovic}\thanks{
         Department of Electrical and Computer Engineering,
         Cornell University;
         \texttt{\{np358,qz16\}@cornell.edu}.} \\
         Cornell University  \\
	 \and
     Sudeep Salgia\thanks{
         Department of Electrical and Computer Engineering,
         Carnegie Mellon University;
         \texttt{ssalgia@andrew.cmu.edu}.}   \\
	 Carnegie Mellon   \\
     \and
	     Qing Zhao\footnotemark[1]\\
         Cornell University  \\
 	} 

\begin{document}

\maketitle

\begin{abstract}
  We consider the problem of contextual kernel bandits with stochastic contexts, where the underlying reward function belongs to a known Reproducing Kernel Hilbert Space. We study this problem under an additional constraint of Differential Privacy, where the agent needs to ensure that the sequence of query points is differentially private with respect to both the sequence of contexts and rewards. We propose a novel algorithm that achieves the state-of-the-art cumulative regret of $\widetilde\cO(\sqrt{\gamma_TT}+\frac{\gamma_T}{\varepsilon_{\textsc{DP}}})$ and $\widetilde\cO(\sqrt{\gamma_TT}+\frac{\gamma_T\sqrt{T}}{\varepsilon_{\textsc{DP}}})$ over a time horizon of $T$ in the joint and local models of differential privacy, respectively, where $\gamma_T$ is the effective dimension of the kernel and $\varepsilon_{\textsc{DP}} > 0$ is the privacy parameter. The key ingredient of the proposed algorithm is a novel private kernel-ridge regression estimator which is based on a combination of private covariance estimation and private random projections. It offers a significantly reduced sensitivity compared to its classical counterpart while maintaining a high prediction accuracy, allowing our algorithm to achieve the state-of-the-art performance guarantees.
\end{abstract}

\section{Introduction}
We study the contextual bandit problem, where a learning agent aims to maximize an unknown reward function $f$ based on noisy observations of the function at sequentially queried points. Specifically, at each time instant $t$, the agent is presented with a context $c_t \in \cC$, which is drawn i.i.d. from a context distribution $\kappa$. Based on the context, the agent takes an action $x_t \in \cX$ and observes $y_t$, a noisy value of the reward $f(c_t,x_t)$. We measure the performance of a learning algorithm $\sA$ using its cumulative regret over a time horizon of $T$, which is defined as follows:

\begin{align}
    \textsc{R}_T(\sA) := \mathbb{E}_{c_1, c_2, \dots, c_T \sim \kappa}\left[\sum_{t=1}^{T} \max_{x\in \cX} f(c_t,x)-f(c_t,x_t)\right],  \label{eqn:er_def}
\end{align}


where the expectation is taken over the context sequence. In this work, we focus on the kernelized setting, i.e., the reward function $f:\cC\times \cX \rightarrow \mathbb{R}$ is assumed to belong to a Reproducing Kernel Hilbert space (RKHS) of a known kernel $k : \cW \times \cW \to \R$, where $\cW = \cC \times \cX$. The kernel-based modeling offers a highly expressive framework for the reward function. In particular, it is known that the RKHS of the Mat\'ern family of kernels, can approximate almost all continuous functions on compact subsets of $\R^d$~\citep{Srinivas}. As a result, kernelized bandits have received significant attention in recent years.


\subsection{Private Kernel Bandits}

In several applications of contextual bandits, the contexts and rewards carry sensitive information that needs to be protected. For example, in recommendation systems, the contexts correspond to the information related to the user to whom the algorithm needs to recommend the products (the action) and the reward captures the user behavior, i.e., whether the user interacted with the recommended products. Evidently, the contexts and rewards contain user-related private information that needs to be systematically protected. This need for systematic guarantees on privacy in contextual bandit applications has fueled the study of contextual bandits under privacy constraints. 

For the problem of contextual bandits, there are two notions of privacy that are widely considered --- {Joint Differential Privacy} (JDP)~\citep{ShariffandSheffet} and {Local Differential Privacy} (LDP)~\citep{Han_et_al_2021}. At a high level, in the JDP setting, the algorithm directly observes the sensitive data, i.e. the context-reward pair, and needs to ensure that the actions taken by the algorithm are differentially private w.r.t. to the data. On the other hand, in the LDP setting, the algorithm gets to observe only a privatized version of the data and use that for future predictions. Evidently, LDP is a stricter notion of privacy than JDP. 




The problems of contextual multi-armed bandits and contextual linear bandits have been studied in detail under both JDP and LDP. On the other hand, our understanding of the problem of private contextual kernel bandits is relatively limited. Existing studies either consider weaker notions of privacy and utility, e.g., only reward-level privacy and/or focus only on simple regret, or limit their attention to a narrow family of structured and highly smooth kernels that do not reflect the inherent hardness of the kernelized setting.
The objective of this work is to develop an algorithm for private contextual kernelized bandits that offers state-of-art cumulative regret performance across a wide family of kernels under both JDP and LDP.

\subsection{Main Results}\label{sec:Main_Results}

We propose a new private contextual kernel bandit algorithm called \textsc{CAPRI}, that achieves state-of-the-art cumulative regret performance across a wide range of kernels under both JDP and LDP. In particular, we demonstrate that \textsc{CAPRI} achieves a cumulative regret of $\widetilde\cO\left(\sqrt{T\gamma_T}+\frac{\gamma_T}{\varepsilon_{\textsc{DP}}}\right)$ and $\widetilde\cO\left(\frac{\gamma_T\sqrt{T}}{\varepsilon_{\textsc{DP}}}+\sqrt{T\gamma_T}\right)$ under the JDP and LDP settings respectively. Here, $\gamma_T$ refers to the effective dimensionality of the kernel and is a measure of the complexity of the kernel. Thus, our guarantees hold for all kernel families with polynomially decaying eigenvalues, which include the widely used Mat\'ern and Squared-Exponential families of kernels. To the best of our knowledge, this is the first work to provide cumulative regret guarantees for the Mat\'ern family of kernels, even under the less stringent constraint of JDP.

Notably, our results recover the order optimal cumulative regret for contextual kernel bandits as $\varepsilon \to \infty$, i.e., in the non-private setting~\citep{Scalett_et_al_2017}. Moreover, for the simplest case of linear kernels, i.e., linear bandits, where $\gamma_T = \cO(d)$, our regret bounds match the lower bound under JDP and achieve the best known bound under LDP~\citep{He_et_al_2022,Han_et_al_2021}.

The kernel-based modeling of the reward function, which allows for infinite-dimensional feature vectors, presents several challenges over the simpler linear reward function model for which the problem is well-understood. Firstly, in any sequential learning problem, it is imperative to add a small layer of privacy at each time step due to the dynamic nature of the dataset. In absence of a careful control, this can result in trivial guarantees on the utility of the algorithm. Under the premise of JDP, the common approach to address this hurdle in linear bandits~\citep{ShariffandSheffet,DubeyPentland2020} is to use the tree-based mechanism~\citep{Dwork}. The tree-based mechanism enables one to construct a private version of the covariance matrix, which is a sum of rank one matrices, and consequently a private predictor of the reward function whose error differs from the non-private version only by a factor growing logarithmically in the number of observed points. On the other hand, the requirement of privatization of the data at the source under LDP preceqludes the use of tree-based mechanism to construct a private covariance matrix. The challenge of privatizing the covariance matrix in the LDP framework is resolved by lower bounding the minimum eigenvalue of the covariance matrix~\citep{Han_et_al_2021,Huang_et_al2024}. Specifically, the authors establish that the smallest eigenvalue of the covariance matrix grows as $\Omega(T/d)$ and thus addition of privatization noise of standard deviation $\cO(\sqrt{T})$ has only a small effect on the predictive performance.

In contextual kernel bandits, these techniques are no longer applicable due to the infinite-dimensional nature of the feature vectors. In particular, the tree-based mechanism leverages the feature space representation of the covariance matrix, which in this case would be infinite-dimensional and hence impractical to work with. On the other hand, the finite-dimensional representation of the covariance matrix using the kernel trick no longer allows for a representation through a sum of rank-one matrices, which prevents the use of the tree-based mechanism. Similarly, an eigenvalue based argument cannot be applied for kernel bandits under the LDP framework because the decaying spectrum of the kernel operator implies that the smallest eigenvalue of the covariance matrix is bounded as $\cO(1)$ and does not grow with the size of the dataset. Another subtle challenge that is unique to the kernel-based setting is that the predictor used to estimate the reward function is also an infinite-dimensional element in the RKHS. In the linear setting, the predictor is a finite-dimensional vector which can be privatized by adding a random vector from the same vector space. Such a technique cannot be extended to the kernel-based setting. Moreover, privatization of all possible outputs would scale the privacy-based error by a factor of $|\cW|$, rendering it trivial.

Our proposed approach presents a significant departure from the commonly used methodology in private linear bandits. The key contribution of this work is a novel, private predictor for the reward function that offers a similar predictive power as its non-private counterpart. In particular, we privatize the covariance matrix by replacing it with a statistically identical and independent copy of the matrix. The independence allows us to obtain privacy for free, while the statistical similarity enables achieving a comparable predictive performance. An important aspect of our algorithm that enables such a construction is the use of reward-independent sampling, i.e., we use actions that do not depend upon the previous rewards. This allows us to draw i.i.d.. copies of context-action pairs to construct the required covariance matrix without requiring feedback from the function. Secondly, to privatize the predictor as an RKHS function, we project the function onto the subspace spanned by finitely many feature vectors, which are drawn at random and are treated as basis vectors for the subspace. This helps approximate the predictor through a finite-dimensional vector consisting of the coefficients the predictor when represented in the randomly chosen basis. The private predictor is obtained by privatizing the coefficients of the original predictor. We show that the `random basis', if chosen correctly, can not only provide a faithful approximation of the original predictor but also ensure that the coefficients have low sensitivity to the data points. This enables  privatization at only a small cost to the predictive performance. We show that these two approaches can be seamlessly integrated to yield our final estimator that has low sensitivity to the underlying dataset and high predictive power. The \textsc{CAPRI} algorithm builds upon this new private estimator by integrating it into an explore-then-eliminate framework.

\subsection{Related Work}
\paragraph{Kernel-based bandits.} The problem of kernel-based bandit optimization has been extensively studied in the non-private setting. Starting with the seminal work of~\citet{Srinivas}, numerous algorithms for kernel-based bandits have been proposed in both contextual~\citep{Valko_et_al_2013} and non-contextual settings~\citep{Gopalan_2017,Batched_Communication}. The optimal performance in the non-private setting is well-understood where several algorithms~\citep{Batched_Communication,Valko_et_al_2013, GP_ThreDS} are known to achieve the order-optimal performance that matches the lower bound~\citep{Scalett_et_al_2017}.

\paragraph{Private Bandit Optimization.} Both multi-armed bandits and linear bandits under privacy constraints have received considerable attention in the literature. In multi-armed bandits,~\citet{Azize_Basu_2022,Azize_et_al_2024} study the problem of private best arm-identification and~\citet{Tenenbaum_et_al_2021} consider cumulative regret minimization under shuffle model of privacy.~\citet{ShariffandSheffet} were the first to study the problem of contextual linear bandits with adversarial contexts under the JDP model. They propose a new algorithm that achieves a cumulative regret of $\cO(\sqrt{T}d^{3/4}/\varepsilon_{\textsc{DP}})$ under a finite arm set assumption.~\citet{Zheng_et_al_2020} consider the analogous problem under the LDP setting and propose an algorithm with regret $\cO(\sqrt{T^{3/4}}/\varepsilon_{\textsc{DP}})$.~\citet{He_et_al_2022} study linear bandits with stochastically generated contexts, under both JDP and LDP. They establish lower bounds under both the privacy models and propose an order-optimal algorithm under JDP with a cumulative regret of $\cO\left(\sqrt{dT}+\frac{d\log(1/\delta_{\textsc{DP}})}{\varepsilon_{\textsc{DP}}}\right)$.~\citet{Han_et_al_2021} propose an algorithm for linear and non-linear contextual bandits under the LDP model that incurs a regret of $\cO(\sqrt{T}d^{3/2}/\varepsilon_{\textsc{DP}})$ under a finite-arm setting.

The problem of private kernel bandits was first studied by~\citet{Kusner_et_al_2015} in the context of hyper-parameter tuning.~\citet{Kharkovskii_et_al_2020} studied private kernel bandits for Squared Exponential kernels under a weaker notion of LDP where the selected actions need to be locally private. Under the restrictive assumption of a diagonally dominant covariance matrix, they propose an algorithm with a simple regret of $\tilde\cO((\varepsilon_{\textsc{DP}}^{-2}+\log T/T)^{1/2})$. ~\citet{Zhou_et_al_2021} consider kernel bandits with heavy-tailed noise, under a weaker notion of privacy where only the rewards need to be private.

The two studies that are closest to ours are those by~\citet{Dubey2021} and~\citet{Pavlovic_et_al_2025}.~\citet{Dubey2021} study the problem of private kernel bandits with adversarial contexts under both JDP and LDP models. They approximate the functions in RKHS using a low-dimensional embedding which allows them to reduce the original problem of that of private linear bandits and circumvent the challenges of kernelized bandits. However, the existence of such a low-dimensional embedding that can approximate all functions in the RKHS is guaranteed only for highly smooth kernels like the Squared Exponential kernel. Consequently, the approach in~\citet{Dubey2021} does not extend beyond the Squared Exponential kernel. The results and methodology proposed in this work hold for a much wider class of kernels, e.g. to the Mat\'ern family of kernels.~\citet{Pavlovic_et_al_2025} consider private kernel bandits with stochastic contexts under JDP. They propose an algorithm that achieves a simple regret of $\widetilde\cO(\sqrt{\frac{\gamma_T}{T}}+\frac{\gamma_T}{T\varepsilon_{\textsc{DP}}})$. In this work, we also consider stochastic contextual kernel bandits under both JDP and LDP and focus on developing an algorithm with low cumulative regret which is more challenging than the metric of simple regret considered in that work. Moreover, our cumulative regret bound implies the simple regret bound obtained in~\citet{Pavlovic_et_al_2025}.

\section{Problem Formulation and Preliminaries}\label{sec:problem_formulation}
\subsection {RKHS, Mercer Theorem, and GP models}

Consider a positive definite kernel $k: \cW \times \cW \to \R$, where $\cW=\cC\times\cX$ is the product of the context and action space. A Hilbert space $\cH_k$ of functions on $\cW$ equipped with an inner product $\ip{\cdot}{\cdot}_{\cH_k}$ is called a Reproducing Kernel Hilbert Space (RKHS) with reproducing kernel $k$ if the following conditions are satisfied: (i) $\forall \ w \in \cW$, $k(\cdot, w) \in \cH_k$; (ii) $\forall \ w \in \cW$, $\forall \ f \in \cH_k$, $f(w) = \ip{f}{k(\cdot, w)}_{\cH_k}$. The inner product induces the RKHS norm, given by $\|f\|_{\cH_k}^2 = \ip{f}{f}_{\cH_k}$. We use $\phi(w)$ to denote $k(\cdot, w)$ and WLOG assume that $k(w,w) = \|\phi(w)\|_{\cH_k}^2 \leq 1$

Let $\zeta$ be a finite Borel probability measure supported on $\cW$ and let $L_2(\zeta,\cW)$ denote the Hilbert space of functions that are square-integrable w.r.t. $\zeta$. Mercer’s Theorem provides an alternative representation for RKHS through the eigenvalues and eigenfunctions of a kernel integral operator defined over $L_2(\zeta,\cW)$ using the kernel $k$.

\begin{theorem}\citep{SVM_Book}\label{Mercer's_Theorem}
Let $\mathcal{W}$ be a compact metric space and $k : \mathcal{W} \times \mathcal{W}\rightarrow \mathbb{R}$ be a continuous kernel. Furthermore, let $\zeta$ be a finite Borel probability measure supported on $\cW$. Then, there exists an orthonormal system of functions $\{\psi_j\}_{j\in \mathbb{N}}$ in $L_2(\zeta,\mathcal{W})$  and a sequence of non-negative values
$\{\lambda_j\}_{j\in \mathbb{N}}$ satisfying $\lambda_1\geq \lambda_2\dots \geq 0$ , such that $k(w,w')=\sum_{j\in \mathbb{N}}\lambda_j\psi_j(w)\psi_j(w')$ holds for all $w,w'\in \mathcal{W}$ and the convergence is absolute and uniform over $w,w'\in \mathcal{W}$. 
\end{theorem}

A commonly used technique to characterize a class of kernels is through their eigendecay profile. 
\begin{definition}\label{assumptio:polynomial_decay}
    Let $\{\lambda_j\}_{j\in \mathbb{N}}$ denote the eigenvalues of a kernel $k$ arranged in the descending order. The kernel $k$ is said to satisfy the polynomial eigendecay condition with a parameter $\beta_p>1$ if, for
    some universal constant $C_p>0$, the relation $\lambda_j\leq C_pj^{-\beta_p}$ holds for all $j \in \N$ .
\end{definition}

We make the following assumptions on the kernel $k$ and the eigenfunctions $\{\psi_j\}_{j \in \N}$, which are necessary to adopt recent 
results for kernelized bandits \citep{Sudeep_Uniform_Sampling,Pavlovic_et_al_2025}.\\

Let $\cM_{\cX} = \{\cX' \subseteq \cX | \text{Unif}(\cX')\}$ be a collection of measures. Here $\text{Unif}(\cX')$ denotes the uniform measure over the set $\cX'$. Let $\cM_{\cW}$ denote the family of measures over subsets of $\cW$, defined as $\cM_{\cW} := \{\varrho(c,x) = \kappa(c) p(x|c)  \ | \ p(x|c) \in \cM_{\cX} \ \forall c \in \cC\}$.

\begin{assumption}\label{assumption:bounded_eigen_functions}
The kernel $k$ satisfies the polynomial eigendecay condition with parameter $\beta_p >1$. Let $\varrho \in \cM_{\cW}$ and $\{\psi_j^{\varrho}\}_{j \in \mathbb{N}}$ denote the corresponding eigenfunctions of the kernel operator corresponding to $\varrho$. There exists $F > 0$ such that $\|\psi_j^{\varrho}(w)\|_{\infty}\leq F$ for all $j\in\mathbb{N}$ and $\varrho \in \cM_{\cW}$.
\end{assumption}

We stress that the assumption (\ref{assumption:bounded_eigen_functions}) bounds eigenfunctions over the support of the measure $\varrho$ , and not the set $\cW$. Assuming bounded eigen-function with respect to a measure is standard practice in both kernel bandit \citep{Uniform_Sampling_Updated_Assumption,Pavlovic_et_al_2025,information_gain_bound} and functional estimation \citep{Maggioni_2008,Mhaskar_et_al_2001,Groechenig_2020, Filbir_et_al_2010} literature. Moreover, we emphasize that the boundedness of the eigen-functions is a crucial assumption in obtaining meaningful information gain bounds for a kernel $k$ \citep{information_gain_bound}.

A Gaussian Process (GP) is a random process $G$ indexed by $\cW$ and is associated with a mean function $\mu : \cW \to \R$ and a positive definite kernel $k : \cW \times \cW \to \R$. The random process $G$ is defined such that for all finite subsets of $\cW$,  $\{w_1, w_2, \dots, w_T\} \subset \cW$, $T \in \mathbb{N}$, the random vector $[G(w_1), G(w_2), \dots, G(w_T)]^{\top}$ follows a multivariate Gaussian distribution with mean vector $[\mu(w_1), \dots, \mu(w_T)]^{\top}$ and covariance matrix $[k(w_i, w_j)]_{i,j=1}^T$. Throughout the work, we consider GPs with $\mu \equiv 0$. When used as a prior for a data generating process under Gaussian noise, the conjugate property provides closed form expressions for the posterior mean and covariance of the GP model. Specifically, given a set of observations $\{\mathbf{W}_T,\mathbf{y}_T\} = \{(w_i,y_i)\}_{i=1}^T$ from the underlying process, the expression for posterior mean and variance of GP model is given as follows:
\begin{align}
    \mu_{T}(w) & =k_{\mathbf{w}_T}(w)^{\top}(\tau \mathbf{I}_T+\mathbf{K}_{\mathbf{W}_T,\mathbf{W}_T})^{-1}\mathbf{y}_T, \label{eqn:posterior_mean}\\
    \sigma^2_T(w)& =k(w,w)-k_{\mathbf{W}_T}^{\top}(w)(\tau\mathbf{I}_T+\mathbf{K}_{\mathbf{W}_T,\mathbf{W}_T})^{-1}k_{\mathbf{W}_T}(w). \label{eqn:posterior_variance}
\end{align}
In the above expressions, $k_{\mathbf{W}_T}(w):=[k(w_1,w),k(w_2,w)\dots k(w_T,w)]^{\top}$, $\mathbf{K}_{\mathbf{W}_T,\mathbf{W}_T} :=[k(w_i,w_j)]_{i,j=1}^{T}$, $\mathbf{I}_T$ is the $T \times T$ identity matrix and $\tau$ is the variance of the Gaussian noise. 

Following a standard approach in the literature~\citep{Srinivas}, we model the data corresponding to observations from the unknown $f$, which belongs to the RKHS of a positive definite kernel $k$, using a GP with the same covariance kernel $k$. In particular, we assume a \emph{fictitious} GP prior over the fixed, unknown function $f$ along with \emph{fictitious} Gaussian distribution for the noise. Such a modeling allows us to predict the values of $f$ and characterize the prediction error through the posterior mean and variance of the GP model.

 \label{gamma_explained}
Lastly, given a set of points $\mathbf{W}_T = \{w_1, w_2, \dots, w_T\} \in \cW$, the information gain of the set $\mathbf{W}_T$ is defined as $\gamma_{\mathbf{W}_T} := \frac{1}{2} \log(\det(\mathbf{I}_T + \tau^{-1}\mathbf{K}_{\mathbf{W}_T,\mathbf{W}_T}))$. Using this, we can define the maximal information gain of a kernel as $\gamma_T := \sup_{\mathbf{W}_T \in \cW^T} \gamma_{\mathbf{W}_T}$. Maximal information gain helps characterize the regret performance of kernel bandit algorithms~\citep{Srinivas,Gopalan_2017}. For the kernels satisfying polynomial eigendecay condition with parameter $\beta_p$, $\gamma_T$ is known to satisfy $\cO(T^{\frac{1}{\beta}_p}\log^{1-\frac{1}{\beta_p}}T)$\citep{information_gain_bound}.

\subsection{Problem Statement}\label{sec:problem_statement}

We consider the problem of private contextual kernel bandits which consists of sequential interaction between a learning algorithm and nature. Under this setting, at the beginning of each time instant $t$, the learning algorithm $\sA$, builds a query strategy $\cQ_t$, which is a (possibly randomized) mapping from the context set $\cC$ to the action set $\cX$. During the time instant $t$, nature draws a context $c_t \in \cC$, i.i.d. from a distribution $\kappa$, plays the action $x_t = \cQ_t(c_t)$, generates a noisy reward $y_t = f(c_t, x_t) + \eta_t$, and reveals $d_t$, a possibly privatized function of $(c_t, y_t)$, to the learning algorithm. Here, $\eta_t$ denotes the zero-mean noise. The algorithm uses the collection $\cD_{t} = \{d_s\}_{s=1}^{t}$ to design the query strategy $\cQ_{t+1}$ for the next time instant. 

We assume that both the action set $\cX \subset \R^{d}$ and the context set $\cC \subset \R^{d'}$ are finite sets with possibly very large cardinality that scales polynomially with $T$ and exponentially with $d$ and $d'$. 
The reward function $f$ belongs to an RKHS $\cH_k$ corresponding to a known kernel $k$ defined over $\cW := \cC \times \cX$ and satisfies $\|f\|_{\cH_k} \leq B$. We make the following assumptions that are commonly adopted in the literature.

\begin{assumption}\label{assumption: bounded_rewards}
    The rewards $\{y_t\}_{t=1}^{T}$ are bounded in absolute value, i.e, $|y_t|<B,\forall t\leq T$ for some $B > 0$.
\end{assumption}

\begin{assumption}
    We have a context generator that can generate i.i.d. contexts according to the distribution $\kappa$.
    \label{assumption:contexts}
\end{assumption}

Assumption~\ref{assumption: bounded_rewards} is commonly adopted in private bandit literature~\citep{DubeyPentland2020,ShariffandSheffet,Han_et_al_2021,Zheng_et_al_2020}. This ensures that an adversary cannot probe an unbounded reward as an input to the algorithm. Assumption~\ref{assumption:contexts} is a commonly adopted assumption in several studies~\citep{Neu_et_al_2024,Zierahn_et_al_2023,Neu_and_Olkhovskaya_2020,Pavlovic_et_al_2025}. Existing studies on stochastic contextual bandits~\citep{Amani_et_al_2023,Hanna_et_al_2022,Hanna_et_al_2023} often assume complete knowledge of context distribution. Assumption~\ref{assumption:contexts} is much milder than having complete knowledge of the context distribution. It can be relaxed to having an access to a static, offline, public database of $\cO(T)$ in-distribution contexts.\\

\subsection{Joint and Local Differential Privacy}

As mentioned earlier, we consider two different notions of privacy in this work --- Joint Differential Privacy (JDP) and Local Differential Privacy (LDP). 

\paragraph{Joint differential privacy.} \label{JDP_sec} This model of differential privacy was first introduced in~\citet{ShariffandSheffet} and extends the classical notion of central differential privacy in offline setting to online setting. We first define the concept of $t$-neighbouring datasets and then use it to formally define JDP. 
\begin{definition}
    Two datasets $\cD_{T} = \{d_s\}_{s = 1}^T, \cD'_{T} = \{d_s'\}_{s = 1}^T$ are said to be $t$-neighbors if $d_t \neq d_t'$ and $d_s = d'_s$ for all $s = \{1,2,\dots, T\} \setminus \{t\}$.
\end{definition}
\begin{definition}\label{JDP_definition}
Let $\{x_t\}_{t=1}^T$ and $\{x_t'\}_{t=1}^T$ respectively denote the actions taken by a randomized contextual bandit algorithm $\sA$ based on datasets $\cD_{T} = \{d_s\}_{s = 1}^T$ and $\cD'_{T} = \{d_s'\}_{s = 1}^T$. The algorithm $\sA$ is said to be $(\varepsilon_{\DP},\delta_{\DP})$-joint differentially private (JDP) if the relation
\begin{align*}
    \Pr(\{x_s\}_{s = t+1}^T = \cB) \leq \exp(\varepsilon_{\DP}) \cdot  \Pr(\{x_s'\}_{s = t+1}^T = \cB) + \delta_{\DP}
\end{align*}
holds for all pairs of $t$-neighboring datasets $\cD_T, \cD'_T$, all subsets $\cB \subset \cX^{T - t}$ and all $t \in \{1,2\dots, T\}$. Here, the probability is taken over the random coins generated by the algorithm.
\end{definition}
\paragraph{Local Differential Privacy.} \label{LDP_privacy} The notion of LDP for online learning problem was introduced by~\citet{Han_et_al_2021} for the problem of linear bandits and is akin to its counterpart in the offline setting~\citep{Kairouz2015Composition}.

\begin{definition}\citep{He_et_al_2022}
    A randomized contextual bandit algorithm $\sA$ is said to be $(\varepsilon_{\textsc{DP}} , \delta_{\textsc{DP}})$ locally differentially private if the dataset $\cD_T = \{d_s\}_{s=1}^T$ satisfies the following relations for all $t \in \{1,2,\dots, T\}$:
    \begin{itemize}
        \item $d_t$ is conditionally independent of $\{(c_s, y_s)\}_{s = 1}^{t-1}$ given $\{d_s\}_{s=1}^{t-1}$ and $(c_t, y_t)$
        \item For all $u,v, w \in \cC \times \R$:
        \begin{align*}
            &\Pr(d_t = w \ | \ \{d_s\}_{s=1}^{t-1}, (c_t, y_t) = u) \leq \exp(\varepsilon_{\DP}) \cdot  \Pr(d_t = w \ | \ \{d_s\}_{s=1}^{t-1}, (c_t, y_t) = v) + \delta_{\DP}.
        \end{align*}
        As before, the probability is taken over the random coins generated by the algorithm.
    \end{itemize}
\end{definition}
LDP ensures that data $(c_t,y_t)$ is privatized immediately (as $d_t$) before being uploaded and used by the algorithm $\sA$. In other words, LDP enforces the privatization of data at the \emph{time it is generated}. In contrast, JDP relaxes this assumption and allows data be privatized at a \emph{time of its use} by the algorithm $\sA$.

\section{Algorithm Description}

In this section, we describe our proposed algorithm, \textsc{CAPRI}, (\textit{\textbf{C}ovariance \textbf{A}pproximation  with \textbf{P}rojected \textbf{R}egress\textbf{I}on}), in both JDP and LDP settings. We first describe the construction of our novel private posterior mean estimator, which is a key component in our algorithm, followed by a detailed description of the algorithm. 

\paragraph{Privatizing the Posterior Mean} As mentioned earlier, the posterior mean based on GP model is a powerful predictor of the function values and is widely used in kernel bandit algorithms to design the query strategy $\cQ_t$. Despite its high predictive power and ubiquitous use in the non-private setting~\citep{Gopalan_2017,Srinivas,Batched_Communication}, the use of posterior mean as outlined in Eqn.~\eqref{eqn:posterior_mean}, presents several layers of challenges in the private setting. Firstly, if $\mu_t$ is used to design $\cQ_{t+1}$, the query strategy at $t+1$, then such a choice results in highly intricate and non-linear relationship between the actions and the observations across different time instants. This cascading effect can make the actions taken by an algorithm highly sensitive to the past ones and consequently challenging to privatize. Secondly, even in the absence of such an intricate relationship between the actions and the observations, the sensitivity of the posterior mean is given by
\begin{align*}
    &\sup_{w} |\mu_t(w) - \mu_t'(w)| =\sup_{w} |\mu_t(w) - f(w) + f(w) - \mu_t'(w)| \\
    &\approx \sup_{w \in \cW} \sigma_t(w) + \sup_{w' \in \cW} \sigma_t'(w).
\end{align*}
Here $\mu_t, \sigma_t$ denote the posterior mean and standard deviation corresponding to dataset $\cD_t$ and $\mu_t', \sigma_t'$ denote the respective values corresponding to $\cD_t'$, a $t$-neighbour of $\cD_t$. Thus, the sensitivity implicitly depends on the rest of the dataset and in the worst-case can be as large as $\Omega(1)$. Lastly, if one were to privatize $\mu$ by separately privatizing $\{\mu(w)\}_{w \in \cW}$, the privacy budget for each point would need to shrink by a factor of $|\cW|$ to maintain the overall privacy guarantees, thereby resulting in trivial utility guarantees. Thus, it is essential to privatize $\mu$ directly as an element of RKHS. As discussed earlier, the infinite-dimensional nature of $\mu$ makes it particularly challenging.


In this work, we develop a novel, privatized version of the posterior mean by systematically addressing each of the three challenges outlined above. Inspired by~\citet{Batched_Communication, Sudeep_Uniform_Sampling}, we also adopt a non-adaptive query strategy, where the actions do not depend on past observations. This allows us to alleviate the first concern. To address the second challenge, we use a modified version of the covariance matrix to build our estimator. In particular, given a dataset $\cD_t = (\bfW_t, \bfy_t)$, we use the estimator 
\begin{align*}
	\tilde{\mu}_{t} & = (\Phi_{\cR_t}\Phi_{\cR_t}^{\top} + \tau \mathbf{Id})^{-1} \Phi_{\bfW_t}\bfy_t,
\end{align*}
where $\cR_t$ is an i.i.d. copy of the original dataset ${\bfW}_t$. The use of a statistically identical copy to estimate the mean offers two advantages. Firstly, it reduces the dependence on the estimator on the dataset. Secondly, the sensitivity of the above estimator depends on the variance corresponding to the points in $\cR_t$ and thus is independent of the original data. More importantly, the statistical similarity between $\cR_t$ and $\bfW_t$ ensures that $\sigma_{\cR_t} \approx \sigma_{\bfW_t}$. This implies that such an estimator has low sensitivity whenever the original dataset has high predictive power --- a scenario that constitutes a typical use case in our algorithm.

A similar technique was also employed in~\citet{Pavlovic_et_al_2025}. Lastly, to privatize $\mu$ as an RKHS function, which is infinite-dimensional, we note that the posterior mean (as an element of the RKHS) can also be written as
\begin{align*}
    \mu_t = \sum_{i = 1}^t \alpha_s \phi(w_i), 
\end{align*}
where $\alpha_i$ is the $i^{\text{th}}$ coordinate of the vector $\boldsymbol{\alpha} = \left(\bfK_{\bfW_T,\bfW_T}+\tau \mathbf{I}_T\right)^{-1}\bfy_t$. Since $\{\phi(w_i)\}_{i=1}^t$ are linearly independent, they form a basis of the subspace $\cH_k^t := \text{span}\{ \phi(w_1), \phi(w_2), \dots, \phi(w_t) \}$. Thus, $\mu_t$ can be represented by a $t$-dimensional vector $\boldsymbol{\alpha}$. If the basis were data-independent, then $\mu_t$ can be privatized by adding noise to $\boldsymbol{\alpha}$. However, we cannot directly use this idea to privatize $\mu_t$ since our basis is also data-dependent. To resolve this bottleneck, we approximate $\mu$ by projecting it onto $\overline{\cH}_k^t = \text{span}\{ \phi(s_1), \phi(s_2), \dots, \phi(s_t) \}$, where $\cS_t = \{s_1, s_2, \dots, s_t\}$ is a random set generated independently from the private dataset . Since the basis elements are now data-independent, we can privatize $\mu$ by simply privatizing the coefficients. We set $\cS_t$ to be a private i.i.d. copy of $\bfW_t$. Such a choice naturally satisfies the independence constraint and also offers a faithful approximation to $\mu_t$, which is essential to ensure high utility for our estimator. 

Based on a combination of all the approaches described above, our proposed private estimator corresponding to the dataset $\cD_t = (\bfW_t, \bfy_t)$ is given as: 
\begin{align}\label{eqn:private_estimator}
	\widehat{\underline{\mu}}_t( \cdot) :=&\bfk^{\top}_{\cS_t}(\cdot)\left(\bfK_{\cS_t,\cR_t}\bfK_{\cR_t,\cS_t}+\tau\bfK_{\cS_t,\cS_t}\right)^{-1/2}\cdot\nonumber\\
    &\cdot \left(\left(\bfK_{\cS_t,\cR_t}\bfK_{\cR_t,\cS_t}+\tau\bfK_{\cS_t,\cS_t}\right)^{-1/2}\bfK_{\cS_r,\bfW_t}\bfy_t+ \bfZ_{\mathrm{priv}}\right)
\end{align}
where $\bfZ_{\mathrm{priv}}$ denotes the privatization noise. Under JDP, $\bfZ_{\mathrm{priv}} \sim \cN(0,\sigma_0^2 \cdot \mathbf{I}_{t} )$ and under LDP $\bfZ_{\mathrm{priv}} \sim \cN(0, t\sigma_0^2 \cdot \mathbf{I}_{t})$, where 
\begin{align}
    \sigma_0 := \widetilde\sigma_{\max}\cdot \frac{4B\log{T}}{\varepsilon_{\textsc{DP}}} \cdot \sqrt{\log\left(\frac{1.25\log{T}}{\delta_{\textsc{DP}}}\right)},
\end{align}
 In the above definition, 
\begin{align}\label{eqn:projected_variance}
    &\tau\widetilde\sigma^2_{\max} :=\max_{w \in \mathrm{supp}(\rho)} k(w,w)-k^{\top}_{\cS}(w)\bfV k_{\cS}(w), \text{   where}\\
    &\bfV=\bfK^{-1}_{\cS,\cS}\bfK_{\cS, \cR}\left(\tau\mathbf{I}+\bfK_{\cR, \cS}\bfK^{-1}_{\cS,\cS}\bfK_{\cS,\cR}\right)^{-1}\bfK_{\cR,\cS}\bfK^{-1}_{\cS,\cS}\nonumber
\end{align}
denotes the maximum posterior variance of points $\cR$ projected onto the set $\cS$ over the support of the sampling measure $\varrho$ that generates $\bfW_t$. For clarity of notation, we use $\widehat{\underline{\mu}}_{t}$ and $\widehat{\underline{\mu}}_{t,\textsc{LDP}}$ to denote the estimators under JDP and LDP setting. The following lemma characterizes the privacy and predictive performance of our proposed estimators.

\begin{lemma}\label{Lemma:main_text_estimator}
    Let $\cD_t = \{\bfW_t, \bfy_t\}$ be a dataset consisting of $t$ points $\bfW_t$ and their corresponding noisy observations $\bfy_t$ from the underlying reward function $f$, such that $\bfW_t$ is independent of the noise vector corresponding to $\bfy_t$. Let $\cR_t$ and $\cS_t$ be two i.i.d. copies of $\bfW_t$ that are independent of both $\bfW_t$ and $\bfy_t$. If $\mu$ and $\mu_{\mathrm{LDP}}$ denote the estimators constructed using $\cD_t, \cR_t$ and $\cS_t$ as outlined in Eqn.~\eqref{eqn:private_estimator}, then $\mu$ is $(\varepsilon_{\DP}/\log{T}, \delta_{\DP}/\log{T})$-JDP and $\mu_{\text{LDP}}$ is  $(\varepsilon_{\DP}, \delta_{\DP})$-LDP w.r.t. to the dataset $\cD_t$. Moreover, the relations
    \begin{align*}
        &\max_{w\in\mathrm{supp}\varrho}|f(w)-\widehat{\underline{\mu}}_t(w)|\le \widetilde\cO\left(\frac{\widetilde\sigma^2_{\max}}{\varepsilon_{\textsc{DP}}}+\widetilde\sigma_{\max}\right)\\
        &\max_{w\in\mathrm{supp}\varrho}|f(w)-\widehat{\underline{\mu}}_{t,\textsc{LDP}}(w)|\le \widetilde\cO\left(\frac{\sqrt{T}\widetilde\sigma^2_{\max}}{\varepsilon_{\textsc{DP}}}+\widetilde\sigma_{\max}\right)
    \end{align*}
    hold with probability $1-\delta$ over the randomness of noise sequence, where $\widetilde\sigma_{\max}$ is as defined in Eqn.~\eqref{eqn:projected_variance}

\end{lemma}
Notably, in the proof of Theorem \ref{theorem:main_theorem} we show that  $\max_{w\in \text{supp}\varrho}\widetilde\sigma^2_t(w)$ can be bounded as $\cO(\frac{\gamma_T}{T})$. This implies that  our error bound scales as $\cO\left(\sqrt\frac{\gamma_T}{T}\right)$ in the JDP setting and $\cO\left(\frac{\gamma_T}{\sqrt{T}}\right)$ in the LDP setting.

\subsection{The \textsc{CAPRI} algorithm}

The \textsc{CAPRI} algorithm is based on the \emph{explore-and-eliminate} philosophy. It operates in epochs that progressively double in length. For each epoch $r$, \textsc{CAPRI} maintains the set of active actions, $\cX_r(c) \subseteq \cX$, corresponding to each context $c \in \cC$. These sets are initialized to $\cX_1(c) = \cX$ for all $c \in \cC$. At the beginning of each epoch $r$, the algorithm has access to two sets $\cR_r$ and $\cS_r$ consisting of $T_r$ context-action pairs, where $T_r$ is the length of the $r^{\text{th}}$ epoch. During the $r^{\text{th}}$ epoch and time instant $t$, the algorithm observes the context $c_t$ and samples $x_t$ uniformly from the set $\cX_r(c_t)$. It is straightforward to note that the points generated during each epoch are independent and identically distributed. It uses $\cR_r$ and $\cS_r$ to construct $\widehat{\underline{\mu}}_r$ as outlined in Eqn.~\eqref{eqn:private_estimator} and update the action sets for the $r+1$-th epoch:
\begin{align}\label{eqn:prunning_sets}
	&\cX_{r+1}(c)= \left\{x\in \cX_{r}(c) \ \bigg| \  \widehat{\underline{\mu}}_r(c,x)\geq \max_{x\in \cX_{r}(c)}\widehat{\underline{\mu}}_r(c,x)-4\Delta_r\right\} \nonumber \\
&\Delta_r :=\beta\left(\frac{\delta}{|\cW|T\log{T}}\right)\widetilde\sigma_{r,\max}+\beta_1\left(\varepsilon_{\textsc{DP}},\delta_{\textsc{DP}},\frac{\delta}{|\cW|T\log{T}}\right)\widetilde\sigma^2_{r,\max}
\end{align}
where $\beta_1=\frac{\log{T}\sqrt{8B\log(\log{T}|\cW|/\delta)\log(1.25\log{T}/\delta_{\textsc{DP}})}}{\varepsilon_{\textsc{DP}}}$, $\beta_{1,\textsc{LDP}}=\sqrt{T_r}\beta_1$ and $\beta(\cdot)$ is the confidence parameter \footnote{Please refer to the supplementary material for a full expression}.
At the end of the $r^{\text{th}}$ epoch, \textsc{CAPRI} constructs the collections $\cR_{r+1}$ and $\cS_{r+1}$, to be used in the next epoch. Both sets consist of i.i.d. $T_{r+1}$ context-action pairs obtained by first drawing $c \sim \kappa$ from the context generator and then drawing the corresponding action $x$ uniformly from the set $\cX_{r+1}(c)$. A pseudo-code of the algorithm can be found in the supplementary material.

\section{Performance Analysis}

The following theorem characterizes the regret and privacy guarantees of the \textsc{CAPRI} Algorithm. 

\begin{theorem}\label{theorem:main_theorem}
    Consider the problem of contextual kernelized bandits described in the Section~\ref{sec:problem_formulation} where Assumptions~\ref{assumption:bounded_eigen_functions}-\ref{assumption:contexts} hold. If \textsc{CAPRI} is run for a time horizon of $T$ with $T \geq T_0$, and privacy parameters $\varepsilon_{\DP} > 0$ and $\delta_{\DP} \in (0,1)$ under JDP (resp. LDP), then 
    \begin{itemize}
        \item \textsc{CAPRI} satisfies $(\varepsilon_{\DP}, \delta_{\DP})$-JDP (resp. LDP)
        \item The regret incurred by the algorithm satisfies:
        \begin{align*}
        &\textsc{R}_{T, \mathrm{JDP}}(\sA)\leq \widetilde{\cO}\left(\sqrt{T \gamma_T }+\frac{\gamma_T\sqrt{\log(1/\delta_{\textsc{DP}})}}{\varepsilon_{\textsc{DP}}}\right)\\
        &\textsc{R}_{T,\mathrm{LDP}}(\sA)\leq \widetilde{\cO}\left(\sqrt{T \gamma_T }+\frac{\gamma_T\sqrt{T\log(1/\delta_{\textsc{DP}})}}{\varepsilon_{\textsc{DP}}}\right)
        \end{align*}
        with probability $1- \delta_{\textsc{ERR}}$ over the randomness in the algorithm and the noise sequence. Here $T_0$ is a constant that depends only on the kernel $k$ and the context distribution
    \end{itemize}
\end{theorem}

As shown in the above theorem, $\textsc{CAPRI}$ guarantees $(\varepsilon_{\textsc{DP}},\delta_{\textsc{DP}})$ differential privacy under both JDP and LDP. Moreover, it achieves a cumulative regret of  $\widetilde{\cO}\left(\sqrt{\gamma_T T}+\frac{\gamma_T}{\varepsilon_{\textsc{DP}}}\right)$ and $\widetilde{\cO}\left(\sqrt{\gamma_T T}+\frac{\gamma_T\sqrt{T}}{\varepsilon_{\textsc{DP}}}\right)$ under JDP and LDP respectively. We would like to point out that for the special case of linear bandits ($\gamma_T = d$), our regret under JDP matches  the lower bound derived in~\citet{He_et_al_2022} and under LDP matches the best known upper bound~\citep{Han_et_al_2021}. The above result also partly resolves the conjecture posed in~\citet{Dubey2021}. Specifically,~\citet{Dubey2021} poses the question about the additional factor of $\cO(T^{1/4})$ in the cumulative regret for the LDP setting when compared to the JDP, for adversarially generated contexts. As shown by the above theorem, this gap in performance between JDP and LDP settings is not present for stochastically generated contexts unlike in the case of adversarially generated contexts. This difference in performance under stochastic and adversarial contextual setting was previously seen in both linear \citep{He_et_al_2022} and generalized-linear bandits~\citep{Han_et_al_2021}. Furthermore, we note that our regret bounds hold under misspecification of the measure $\kappa$. Namely for a misspecified measure $\kappa'$, for which $ \forall c, \frac{\kappa'(c)}{\kappa(c)}\in \left(1-u,1+u\right)$ for some $u = \cO(\sqrt{\gamma_T/T})$, Theorem \ref{theorem:main_theorem} still remains true. We refer the reader to supplementary material  for a detailed proof of the theorem. 

\section{Conclusion}


In this work, we study the problem of private contextual kernel bandits with stochastically generated contexts. We propose a novel private estimator for RKHS functions that is based on a careful combination of covariance approximation and random projection in the RKHS space. Our estimator simultaneously offers low sensitivity to the underlying dataset and order-optimal prediction accuracy. Building upon our novel estimator, we also proposed a new algorithm called \textsc{CAPRI} that achieves state-of-the-art cumulative regret performance across a wide range of kernels under both JDP and LDP setting.  
We conjecture that our proposed algorithm offer optimal regret performance under both JDP and LDP settings. Establishing the lower bounds on regret in private settings is an interesting future direction to explore.



\newpage

\bibliography{sample.bib}

\newpage

\onecolumn

\section{Appendix A. Performance and Privacy Guarantees}\label{App:A}

We first state several results frequently utilized in kernel-based optimization literature.

\begin{lemma}\label{variance_bounding}
   Consider a measure $\varrho$ supported on a $\cW=\cC\times \cX$ defined in assumption \ref{assumption:bounded_eigen_functions} and let $\bfZ=T\bf\Lambda+\tau\bf Id$ where $\Lambda= \mathbb{E}_{w\sim \varrho}[\phi(w)\phi^{\top}(w)]$. For $T>\overline{T}$:
   \begin{align*}
       \sup_{x\in\mathrm{supp}\varrho}\phi^{\top}(w)\bfZ^{-1}\phi(w)\leq \frac{54F^2}{13}\frac{\gamma_T}{T}
   \end{align*}
   Where $\overline{T}$ is a measure $\varrho$ and kernel $k$ dependent constant defined in \cite{Sudeep_Uniform_Sampling}.
\end{lemma}

The result above is a direct consequence of Lemma 3.4 in \cite{Sudeep_Uniform_Sampling}. Assumption \ref{assumption:bounded_eigen_functions} ensures that the eigenfunctions are bounded over the support of $\varrho$ , and thus results from \cite{Sudeep_Uniform_Sampling} can be applied.\\

\begin{lemma}(Adapted after \cite{Calandriello_Sketching})\label{fact1}
    Consider a covariance matrix $\Phi_{\cS}\Phi^{\top}_{\cS}+ \tau\mathbf{Id}$, where $\cS\subset \cW$ and a linear operator  $\bfA+ \tau\mathbf{Id}:\cH_k \rightarrow \cH_k$. Suppose for some $\alpha \in(0,1)$ we have $(1-\alpha)(\bf{A}+ \tau\mathbf{Id})\preceq\Phi_{\cS}\Phi^{\top}_{\cS}+ \tau\mathbf{Id}\preceq (1+\alpha)(\bf{A}+ \tau\mathbf{Id})$. 
    The following holds for the projected operator $\cP_{\cS}\bf{A}\cP_{\cS}$:
    \begin{align*}
       \frac{1-\alpha}{1+\alpha}(\bfA+\tau \mathbf{Id}) \preceq\cP_{\cS}\bfA\cP_{\cS}+ \tau\mathbf{Id} \preceq  \frac{1+\alpha}{1-\alpha}(\bfA+\tau \mathbf{Id}) 
    \end{align*}

    Here, the projection operator is defined as $\cP_{\cS}=\Phi_{\cS}(\Phi^{\top}_{\cS}\Phi_{\cS})^{-1}\Phi^{\top}_{\cS}$.
\end{lemma}
\begin{proof}
We can write:
\begin{align*}
&\cP_{\cS}\bfA\cP_{\cS}\succeq\frac{1}{1+\alpha}\cP_{\cS}(\bf\Phi_{\cS}\bf\Phi^{\top}_{\cS}-\alpha\tau \mathbf{Id})\cP_{\cS}=\frac{1}{1+\alpha}(\bf\Phi_{\cS}\bf\Phi^{\top}_{\cS}-\alpha\tau\cP_{\cS})\succeq\\
&\succeq \frac{1}{1+\alpha}((1-\alpha)(\bfA+ \tau\mathbf{Id})-(\alpha+1) \tau\mathbf{Id})=\frac{1-\alpha}{1+\alpha}(\bfA+\tau\mathbf{Id})-\tau\mathbf{Id}
\end{align*}
Adding $\alpha \mathbf{Id}$ to both sides we have $\cP_{\cS}\mathbf{A}\cP_{\cS}+\tau\mathbf{Id}\succeq \frac{1-\alpha}{1+\alpha}(\bfA+\tau\mathbf{Id})$. Symmetric derivation holds for the right-hand side\\
\end{proof}

\begin{lemma}(Adapted after \cite{Calandriello_Sketching})\label{Fact 2.}
    Consider the projection operator $\cP_{\cS}$ induced by the columns of $\Phi_{\cS}$ where $\cS\subset \cW$, and a linear operator $\bfA+\tau\mathbf{Id}:\cH_k\rightarrow\cH_k$ .  If for some $\alpha \in(0,1)$ we have $(1-\alpha)(\bfA+ \tau\mathbf{Id})\preceq\Phi_{\cS}\Phi^{\top}_{\cS}+ \tau\mathbf{Id}\preceq (1+\alpha)(\bfA+ \tau\mathbf{Id})$, then the following identity holds:
    \[
    (\mathbf{Id}- \cP_{\cS})\preceq \frac{\tau}{1-\alpha}(\bfA+\mathbf{Id})^{-1}
    \]
\end{lemma}
\begin{proof}
    \begin{align*}
        &(\mathbf{Id}- \cP_{\cS})\preceq \mathbf{Id}-\Phi_{\cS}(\Phi^{\top}_{\cS}\Phi_{\cS}+\tau\mathbf{I}_{\cS})^{-1}\Phi^{\top}_{\cS}\preceq \mathbf{Id}-\Phi_{\cS}\Phi^{\top}_{\cS}(\Phi_{\cS}\Phi^{\top}_{\cS}+ \tau \mathbf{Id})^{-1}\preceq\\
        &\preceq \tau (\Phi_{\cS}\Phi^{\top}_{\cS}+ \tau \mathbf{Id})^{-1}\preceq \frac{\tau}{1-\alpha}(\mathbf{A}+\tau \mathbf{Id})^{-1}
    \end{align*}
    The first inequality follows from $(\Phi^{\top}_{\cS}\Phi_{\cS}+\tau\mathbf{Id})^{-1}\preceq (\Phi^{\top}_{\cS}\Phi_{\cS})^{-1}$.
\end{proof}

\begin{lemma}(Kernel trick matrix identity)\label{kernel_trick_mat_identity}
Consider two matrices $\bfA,\bfB$ so that $\bfA\bfB+\tau \mathbf{Id}$ is invertible. The following identity holds:
\begin{align*}
    \bfA(\bfB\bfA+\tau \mathbf{Id})^{-1}=(\bfA\bfB+\tau\mathbf{Id})^{-1}\bfA
\end{align*}
\end{lemma}
\begin{proof}
Note that:
    \begin{align*}
        \bfA(\bfB\bfA+ \tau\mathbf{Id})=(\bfA\bfB+ \tau\mathbf{Id})\bfA
    \end{align*}
$\bfA\bfB+\tau \mathbf{Id}$ is invertible and as $\bfA\bfB, \bfB\bfA$ have the same spectrum, so is $\bfB\bfA+\tau \mathbf{Id}$. Multiplying both sides of the equality by $(\bfB\bfA+\tau \mathbf{Id})^{-1}$ from the right, and $(\bfA\bfB+\tau \mathbf{Id})^{-1}$ from the left  gives the desired identity.
\end{proof}

We next show that a randomly drawn covariance matrix is sufficiently close to the expected operator $\bfZ=T\mathbf{\Lambda}+ \tau \mathbf{Id}$. This result will be essential in showing that projection onto an in-distribution "copy"  of the data-set, keeps the same order of utility. We will also use Lemma(\ref{Lemma:covariance_app}) to show that that the covariance of the data-set can be estimated by an in distribution copy.

\begin{lemma}\label{Lemma:covariance_app}
    Suppose $T>\max\left(\overline{T},\left(\frac{180\log(14T/\delta)}{F^2}\right)^{\beta_p/\beta_p-1}\right)$ points $\cR=\{r_1,r_2, \dots r_T\}$ are sampled i.i.d from $\cW=\cC\times\cX$ according to the measure $\varrho$, satisfying assumption \ref{assumption:bounded_eigen_functions}. Consider the operator $\bfZ= \tau\mathbf{Id}+T\mathbf{\Lambda}$ where $\Lambda=\mathop{\mathbb{E}}_{x\sim\varrho}[\phi(w)\phi^{\top}(w)]$ and let
    $\mathbf{R}=\sum_{i=1}^{T} \phi(r_i)\phi(r_i)^{\top}+\tau\mathbf{Id}$. We claim with probability at least $1-\delta$:
        \begin{align*}
            \left\|\bfZ^{-1/2}
        \bfR\bfZ^{-1/2}-\mathbf{Id}\right\|_2\leq 2\sqrt{C_0\log(14T/\delta)}+ \frac{4}{3}\log(14T/\delta)C_0
        \end{align*}
    Where $C_0=\sup_{w\in \mathrm{supp} \rho}\phi^{\top}(w)\bfZ^{-1}\phi(w)$ is the maximum variance of the expected covariance operator over the support of the measure $\varrho$.
    
\end{lemma}

\begin{proof}
   We will use Bernstein concentration inequality for self-adjoint Hilbert Space operators. We start by re-writing  the expression as:
    \begin{align*}
        \bfZ^{-1/2}\bfR\bfZ^{-1/2}-\mathbf{Id}&=\sum_{i=1}^{T} \left(\bfZ^{-1/2}\phi(r_i)\phi^{\top}(r_i)\bfZ^{-1/2}-\bfZ^{-1/2}\mathbf{\Lambda}\bfZ^{-1/2}\right)=\\
        &= \sum_{i=1}^{T}\left(V_i-\bfZ^{-1/2}\Lambda\bfZ^{-1/2}\right)
    \end{align*}
First note that $\mathbb{E}\left[V_i-\bfZ^{-1/2}\Lambda\bfZ^{-1/2}\right]=0$. We can bound the spectrum as:
\begin{align*}
    U =\sup_{r_i} \|V_i-\bfZ^{-1/2}\Lambda\bfZ^{-1/2}\|_2&\leq \max(\sup_{r_i}\|V_i\|_2,\|\bfZ^{-1/2}\Lambda\bfZ^{-1/2}\|_2) \leq \max\left(\sup_{r} \phi^{\top}(r)\bfZ^{-1}\phi(r), \frac{1}{T}\right)
\end{align*}
Where the first inequality follows as both matrices as symmetric positive semi-definite. Introduce the shorthand $C_0=\sup_{r} \phi^{\top}(r)\bfZ^{-1}\phi(r)$. We next bound the sum of variances:
\begin{align*}
\sigma^2&=\left\|\sum^{T}_{i=1}\mathbb{E}\left[(V_i-\bfZ^{-1/2}\Lambda\bfZ^{-1/2})^2\right]\right\|_2=\\ &=\left\|\sum^{T}_{i=1}\mathbb{E}\left[V_i^2\right]-(\bfZ^{-1/2}\Lambda\bfZ^{-1/2})^2\right\|_2\leq T\left\|\mathbb{E}_{r\sim\varrho}\left[\bfZ^{-1/2}\phi(r)\phi^{\top}(r)\bfZ^{-1}\phi(r)\phi^{\top}(r)\bfZ^{-1/2}\right]\right\|_2\leq\\
    &\leq TC_0 \left\|\mathbb{E}_{r\sim\varrho}\left[\bfZ^{-1/2}\phi(r)\phi^{\top}(r)\bfZ^{-1/2}\right]\right\|_2=TC_0\left\|\bfZ^{-1/2}\Lambda\bfZ^{-1/2}\right\|_2\leq C_0
\end{align*}
The first inequality follows as $V_i$'s are iid and the operators $\mathbb{E}\left[\left(V_i-\bfZ^{-1/2}\Lambda\bfZ^{-1/2}\right)^2\right]$ are positive semi-definite, while the second is simply the result of upper bounding $\phi^{\top}(r)\bfZ^{-1}\phi(r)\leq C_0$. \\
The final inequality follows as $\bfZ\succeq T \Lambda$. We now apply operator Bernstein inequality\citep{Minsker_Bernstein,Calandriello_Sketching}:
\begin{align*}
    P\left(\left\|\sum_{i=1}^T (V_i-\bfZ^{-1/2}\Lambda\bfZ^{-1/2})\right\|_2\geq t\right)\leq 14T\exp\left(-\frac{t^2/2}{\sigma^2+tU/3}\right)
\end{align*}

Choosing $t=2\sqrt{C_0\log(14T/\delta)}+ \frac{4}{3}\log(14T/\delta)\max\left(C_0,\frac{1}{T}\right)$ we have w.p  at least $1-\delta$:
\begin{align*}
    \left\|\sum_{i=1}^T (V_i-\bfZ^{-1/2}\Lambda\bfZ^{-1/2})\right\|_2\leq 2\sqrt{C_0\log(14T/\delta)}+ \frac{4}{3}\log(14T/\delta)\max\left(C_0,\frac{1}{T}\right) 
\end{align*}
By using Lemma \ref{variance_bounding} and the assumed bound on $T$ this expression can be simplified as:
\begin{align*}
    \left\|\sum_{i=1}^T (V_i-\bfZ^{-1/2}\Lambda\bfZ^{-1/2})\right\|_2\leq 2\sqrt{C_0\log(14T/\delta)}+ \frac{4}{3}\log(14T/\delta)C_0
\end{align*}

\end{proof}
We note that the threshold $T>\max\left(\overline{T},\left(\frac{180\log(14T/\delta)}{F^2}\right)^{\beta/\beta-1}\right)$ was chosen so that $2\sqrt{C_0\log(14T/\delta)}> \frac{4}{3}\log(14T/\delta)C_0$ and $2\sqrt{C_0\log(14T/\delta)}+ \frac{4}{3}\log(14T/\delta)C_0<\frac{1}{9}$. This was namely done for the ease of presentation and the only key requirement here is that $T>\overline{T}$.

Next, we introduce a novel estimator that builds upon the posterior mean by both estimating the covariance matrix and projecting the data set onto randomly sampled feature directions. We show that this estimator, despite having a significantly weaker dependence on the data-set, still keeps the order of approximation error of the posterior-mean.\\


\begin{lemma}\label{Lemma:Appendix_estimator}
    Let $\cD_t = \{\bfW_t, \bfy_t\}$ be a dataset of consisting of $t$ points $\bfW_t$ generated via the measure $\varrho$ and their corresponding noisy observations $\bfy_t$ from the underlying reward function $f$, such that $\bfW_t$ is independent of the noise vector corresponding to $\bfy_t$. Let $\cR_t$ and $\cS_t$ be two i.i.d. copies of $\bfW_t$ that are independent of both $\bfW_t$ and $\bfy_t$. If $\mu$ and $\mu_{\mathrm{LDP}}$ denote the estimators constructed using $\cD_t, \cR_t$ and $\cS_t$ as :
    \begin{align*}
        	&\widehat{\underline{\mu}}_t( \cdot) :=\bfk^{\top}_{\cS_t}(\cdot)\left(\bfK_{\cS_t,\cR_t}\bfK_{\cR_t,\cS_t}+\tau\bfK_{\cS_t,\cS_t}\right)^{-1/2}\cdot\left(\left(\bfK_{\cS_t,\cR_t}\bfK_{\cR_t,\cS_t}+\tau\bfK_{\cS_t,\cS_t}\right)^{-1/2}\bfK_{\cS_r,\bfW_t}\bfy_t+ \bfZ_{\mathrm{priv}}\right)\\
            &\widehat{\underline{\mu}}_{t,\text{LDP}}( \cdot) :=\bfk^{\top}_{\cS_t}(\cdot)\left(\bfK_{\cS_t,\cR_t}\bfK_{\cR_t,\cS_t}+\tau\bfK_{\cS_t,\cS_t}\right)^{-1/2}\cdot\left(\left(\bfK_{\cS_t,\cR_t}\bfK_{\cR_t,\cS_t}+\tau\bfK_{\cS_t,\cS_t}\right)^{-1/2}\bfK_{\cS_r,\bfW_t}\bfy_t+ \bfZ_{\mathrm{priv}, \mathrm{LDP}}\right)
    \end{align*}
   where $\bfZ_{\mathrm{priv}} \sim \cN(0,\sigma_0^2 \cdot \mathbf{I}_{t} )$ and under LDP $\bfZ_{\mathrm{priv},\mathrm{LDP}}\sim \cN(0, t\sigma_0^2 \cdot \mathbf{I}_{t})$, 
\begin{align}
    \sigma_0 := \widetilde\sigma_{\max}\cdot \frac{4B\log{T}}{\varepsilon_{\textsc{DP}}} \cdot \sqrt{\log\left(\frac{1.25\log{T}}{\delta_{\textsc{DP}}}\right)},
\end{align} and $\widetilde\sigma_{\max}=\max_{w\in \mathrm{supp}\varrho}\phi^{\top}(w)\left(\cP_{\cS_t}\Phi_{\cR_t}\Phi^{\top}_{\cR_t}\cP_{\cS_t}+\tau\mathbf{I}\right)^{-1}\phi(w)$ then $\mu$ is $(\varepsilon_{\DP}/\log{T}, \delta_{\DP}/\log{T})$-JDP and $\mu_{LDP}$ is  $(\varepsilon_{\DP}, \delta_{\DP})$-LDP w.r.t.the dataset $\cD_t$. Moreover w.p. at least $1-\delta$ the following relations hold:
\begin{align*}
&\max_{w\in \mathrm{supp}\varrho}|f(w)-\widehat{\underline{\mu}}_t(w)|\leq\beta\left(\frac{\delta}{|\cW|}\right)\widetilde\sigma_{\mathrm{max}}+\beta_1\left(\varepsilon_{\mathrm{DP}},\delta_{\mathrm{DP}},\frac{\delta}{|\cW|}\right)\widetilde\sigma^2_{\mathrm{max}}\\
&\max_{w\in \mathrm{supp}\varrho}|f(w)-\widehat{\underline{\mu}}_{t,\text{LDP}}(w)|\leq\beta\left(\frac{\delta}{|\cW|}\right)\widetilde\sigma_{\mathrm{max}}+\beta_{1,\mathrm{LDP}}\left(\varepsilon_{\mathrm{DP}},\delta_{\mathrm{DP}},\frac{\delta}{|\cW|}\right)\widetilde\sigma^2_{\mathrm{max}}
\end{align*}
where $\beta(\delta)=\left(90B\sqrt{\log\left(\frac{168T}{\delta}\right)}+\frac{52B\sqrt{\log\left(\frac{168T}{\delta}\right)\log\left(\frac{12}{\delta}\right)}}{\sqrt{\tau}}+3B\sqrt{2\log\left(\frac{6}{\delta}\right)}+ \sqrt{24\tau}\right),\beta_1(\varepsilon_{\mathrm{DP}},\delta_{\mathrm{DP}},\delta)=\frac{8B\log{T}}{\varepsilon_{\textsc{DP}}}\log\left(\frac{3}{\delta}\right)\sqrt{\log\left(\frac{1.25\log{T}}{\delta_{\textsc{DP}}}\right)}$ and $\beta_{1,\mathrm{LDP}}=\sqrt{t}\beta_{1,\mathrm{LDP}}$.
\end{lemma}

\begin{proof}

    Note that by assumption \ref{assumption: bounded_rewards} noise instances $\{\eta_i\}_{i=1}^{T}$ are in $[-B,B]$ and thus 
    $\eta_i$ is $B^2$-sub Gaussian random variable \citep{Wainwright_2019}.
    We will divide the proof in  several steps. Note that the analysis for LDP and JDP is entirely the same,  the only exception being the scaling of the error function due to the extra privatization noise $\bfZ_{\mathrm{priv},\mathrm{LDP}}$. Having this in mind, we will present the derivation for the JDP estimator $\widehat{\underline{\mu}}_t(\cdot)$ and outline the steps for $\widehat{\underline{\mu}}_{t, \mathrm{LDP}}(\cdot)$

    The total approximation error can be written as:
    \begin{align*}
        |f(w)-\widehat{\underline{\mu}}_t(w)|\leq \underbrace{|f(w)-\widetilde\mu_t(w)|}_{E_1}+\underbrace{|\widetilde\mu_t(w)-\widehat\mu_t(w)|}_{E_2}+\underbrace{|\widehat{\underline{\mu}}_{t}-\widehat\mu_t(w)|}_{E_3}
    \end{align*}
    Where $\widetilde\mu_t, \widehat\mu_t$ are "intermediate" estimators to $\widehat{\underline{\mu}}_t$:
    \begin{align}
        &\widetilde\mu_t(\cdot)=\phi^{\top}(\cdot)(\cP_{\cS}\Phi_{\bfW_t}\Phi^{\top}_{\bfW_t}\cP_{\cS}+\tau \mathbf{Id})^{-1}\cP_{\cS}\Phi_{\bfW_t}\bfy_t\\
        &\widehat{\mu}_t(\cdot)=\phi^{\top}(\cdot)(\cP_{\cS}\Phi_{\cR}\Phi^{\top}_{\cR}\cP_{\cS}+\tau \mathbf{Id})^{-1}\cP_{\cS}\Phi_{\bfW_t}\bfy_t
    \end{align}
 We will separately bound $E_1,E_2,E_3$ and then finally prove the privacy claim in the end. For the ease of presentation, We drop the $t$ subscript from the in-distribution copies of the data-set, $\cS_t, \cR_t$.\\

 \textbf{Step 1: Bounding $E_1$.}\\
 
 In bounding $E_1$ we closely follow the derivation in \cite{Vakili_Suddep_Sketching} with appropriate changes to account for different approach to exploration and the design of the projection set.
 We first separate the error in the noise and the approximation component. \\

  For notational convenience introduce the short-hand $\hatZ_{\cS}=\sum_{i=1}^{t}\cP_{\cS}\phi(w_i)\phi^{\top}(w_i)\cP_{\cS}+ \tau\mathbf{Id}$.
    \begin{align*}
    &|f(w)-\phi^{\top}(w)\hatZ_{\cS}^{-1}\cP_{\cS}\Phi_{\bf{W}_t}y_T|\leq |\phi^{\top}(w)\hatZ_{\cS}^{-1}\cP_{\cS}\Phi_{\bfW_t}\Phi_{\bfW_t}f|+|\phi^{\top}(w)\hatZ_{\cS}^{-1}\cP_{\cS}\Phi_{\bfW_t}\eta_{1:t}|
\end{align*}
We further bound the first term as :
\begin{align}
    &|\phi^{\top}(w)(\mathbf{Id}-\hatZ_{\cS}^{-1}\cP_{\cS}\Phi_{\bfW_t}\Phi^{\top}_{\bfW_t})f|\leq \|\phi^{\top}(w)(\mathbf{Id}-\hatZ^{-1}_{\cS}\cP_{\cS}\Phi_{\bfW_t}\Phi^{\top}_{\bfW_t})\|_{\cH_k}\\
    &\leq \|\phi^{\top}(w)\hatZ^{-1}_{\cS}(\cP_{\cS}\Phi_{\bfW_t}\Phi^{\top}_{\bfW_t}\cP_{\cS}+ \tau\mathbf{Id}-\cP_{\cS}\Phi_{\bfW_t}\Phi^{\top}_{\bf{W}_t})\|\leq \\
    &\leq\tau \|\phi^{\top}(w)\hatZ^{-1}_{\cS}\|+\|\phi^{\top}(w)\hatZ^{-1}_{\cS}\cP_{\cS}\Phi_{\bfW_t}\Phi^{\top}_{\bfW_t}(\cP_{\cS}- \mathbf{Id})\|\leq\\
    &\leq \tau \|\phi^{\top}(w)\hatZ^{-1}_{\cS}\|+\sqrt{\phi^{\top}(w)\hatZ^{-1}_{\cS}\cP_{\cS}\Phi_{\bfW_t}\Phi^{\top}_{\bfW_t}(\mathbf{Id}-\cP_{\cS})\Phi_{\bfW_t}\Phi^{\top}_{\bfW_t}\cP_{\cS}\hatZ^{-1}_{\cS}\phi(w)}
\end{align}
From Lemma \ref{Lemma:covariance_app} it follows with probability at least $1-\delta/2$ that $\max(\|\bfZ^{-1/2}(\Phi_{\cS}\Phi^{\top}_{\cS}+ \tau\mathbf{Id})\bfZ^{-1/2}-\mathbf{Id}\|_2,\|\bfZ^{-1/2}\hatZ\bfZ^{-1/2}-\mathbf{Id}\|)\leq \frac{1}{9}$. This  implies that all of the eigen-values of the two operators $\bfZ^{-1/2}(\Phi_{\cS}\Phi^{\top}_{\cS}+ \tau\mathbf{Id})\bfZ^{-1/2},\bfZ^{-1/2}\hatZ\bfZ^{-1/2}$ are in range $[\frac{8}{9}, \frac{10}{9}]$. By multiplying both operators with $\bfZ^{1/2}$ from left and right we can write:
\begin{align*}
    \frac{6}{8} \hatZ\preceq\frac{8}{10} \hatZ \preceq\frac{8}{9}\bfZ\preceq \Phi_{\cS}\Phi^{\top}_{\cS}+\tau\mathbf{Id}\preceq \frac{10}{9}\bfZ\preceq \frac{10}{8}\hatZ
\end{align*}

By Lemma \ref{Fact 2.} applied on operators $\hatZ, \Phi_{\cS}\Phi^{\top}_{\cS}+\tau\mathbf{Id}$ we have  $\mathbf{Id}-\cP_{\cS}\preceq \frac{4}{3} \hatZ^{-1}$. By Lemma \ref{fact1} , $\phi^{\top}(w)\hatZ^{-1}_{\cS}\phi(w)\leq \frac{5}{3} \phi^{\top}(w)\hatZ^{-1}\phi(w)\leq \frac{50}{27} \phi^{\top}(w)\bfZ^{-1}\phi(w)$. Now applying the same derivation as in \cite{Vakili_Suddep_Sketching} we have:
\begin{align*}
    |\phi^{\top}(w)(\mathbf{Id}-\hatZ_{\cS}^{-1}\cP_{\cS}\Phi_{\bfW_t}\Phi^{\top}_{\bfW_t})f|\leq \sqrt{\frac{8\tau}{3}}\sqrt{\phi^{\top}(w) \bfZ^{-1}_{\cS} \phi^{\top}(w)}\leq \sqrt{\frac{8\tau}{3}}\sqrt{\frac{50}{27}\phi^{\top}(w) \mathbf{Z}^{-1} \phi(w)}
\end{align*}
To bound the noise component, it suffices to bound the $\|\phi^{\top}(w)\hatZ_{\cS}^{-1}\cP_{\cS}\Phi_{\bfW_t}\|_2$. Vector $\{\eta_{1:t}\}$ is $B^2$-sub Gaussian, and the rest of the derivation will follow from Chernoff-Hoeffding inequality \citep{Wainwright_2019}.\\
Note that by identical argument as in \cite{Vakili_Suddep_Sketching}, we have:
\begin{align*}
    \|\phi^{\top}(w)\hatZ_{\cS}^{-1}\cP_{\cS}\Phi_{\bfW_t}\|_2\leq \sqrt{\phi^{\top}(w)\bfZ^{-1}_{\cS}\phi(w)}\leq \sqrt{\frac{50}{27}\phi^{\top}(w)\bfZ^{-1}\phi(w)} 
\end{align*}
Thus w.p at least $1-\delta/2$
\begin{align*}
    \phi^{\top}(w)\hatZ_{\cS}^{-1}\cP_{\cS}\Phi_{\bfW_t}\eta_{1:t}\leq B\sqrt{2\log(2/\delta)\frac{50}{27}\phi^{\top}(w)\bfZ^{-1}\phi(w)}
\end{align*}
Finally w.p. probability at least $1- \delta$:
\begin{align*}
    E_1\leq \sqrt{\frac{50}{27}\phi^{\top}(w)\bfZ^{-1}\phi(w)}\left(B\sqrt{2\log(2/\delta)}+ \sqrt{\frac{8\tau}{3}}\right)
\end{align*}

\textbf{Step 2: Bounding $E_2$.}\\

For notational convenience let $\tildeZ=\Phi_{\cR}\Phi^{\top}_{\cR}+\tau \mathbf{Id}$ and in general add a subscript $\cS$ to denote  a space $\Phi_{\cS}$  projected operator. Specifically let $\bfZ_{\cS}=\cP_{\cS}T\bf\Lambda\cP_{\cS}+\tau\mathbf{Id}$ and $\tildeZ_{\cS}=\cP_{\cS}\Phi_{\cR}\Phi^{\top}_{\cR}\cP_{\cS}+\tau \mathbf{Id}$.
\begin{align}\label{eq_3}
    &|\widetilde\mu_t(w)-\widehat\mu_t(w)|\leq |\phi^{\top}(w)\mathbf{\widehat{Z}}^{-1}_{\cS}\cP_{\cS}\Phi_{\bfW_t}\bfy_t-\phi^{\top}(w)\mathbf{\widetilde{Z}}_{\cS}^{-1}\cP_{\cS}\Phi_{\bfW_t}\bfy_t|\leq \\
    &\leq  |\phi^{\top}(w)\bfZ^{-1}_{\cS}(\bfZ_{\cS}-\hatZ_{\cS})\hatZ^{-1}_{\cS}P_{\cS}\Phi_{\bfW_t}\bfy_t|+|\phi^{\top}(w)\tildeZ^{-1}_{\cS}(\tildeZ_{\cS}-\bfZ_{\cS})\bfZ^{-1}_{\cS}P_{\cS}\Phi_{\bfW_t}\bfy_t|
\end{align}

We will bound each of the two summands in Eq.(\ref{eq_3}) separately. By noting that $y_t=\phi^{\top}(w_t)f+\eta_t$, we can separate the approximation and the noise related error:
\begin{align}\label{eq_2}
    &|\phi^{\top}(w)\bfZ^{-1}_{\cS}(\bfZ_{\cS}-\hatZ_{\cS})\hatZ^{-1}_{\cS}P_{\cS}\Phi_{\bfW_t}y_t|\leq \\
    & \leq |\phi^{\top}(w)\bfZ^{-1}_{\cS}(\bfZ_{\cS}-\hatZ_{\cS})\hatZ^{-1}_{\cS}P_{\cS}\Phi_{\bfW_t}\Phi^{\top}_{\bfW_t} f|+|\phi^{\top}(w)\bfZ^{-1}_{\cS}(\bfZ_{\cS}-\hatZ_{\cS})\hatZ^{-1}_{\cS}P_{\cS}\Phi_{\bfW_t}\eta_{1:t}|
\end{align}
After some algebraic manipulation on the first term we arrive to:
\begin{align}
    &|\phi^{\top}(w)\bfZ^{-1}_{\cS}(\bfZ_{\cS}-\hatZ_{\cS})\hatZ^{-1}_{\cS}P_{\cS}\Phi_{\bfW_t}\Phi^{\top}_{\bfW_t}f|=\\
    &=|\phi^{\top}(w)\bfZ^{-1}_{\cS}P_{\cS}(\bfZ-\hatZ)P_{\cS}\hatZ^{-1}_{\cS}P_{\cS}\Phi_{\bfW_t}\Phi^{\top}_{\bfW_t}f|\leq  \\
    &\leq \|\bfZ^{-1}_{\cS}P_{\cS}(\bfZ-\hatZ)P_{\cS}\|_2\|\hatZ^{-1}_{\cS}P_{\cS}\Phi_{\bfW_t}\Phi^{\top}_{\bfW_t}f\|_2=\\
    &=\|P_{\cS}\bfZ^{-1}_{\cS}P_{\cS}\bfZ(\mathbf{Id}-\bfZ^{-1}\hatZ)\|_2\|\hatZ^{-1}_{\cS}P_{\cS}\Phi_{\bfW_t}\Phi^{\top}_{\bfW_t}f\|_2\leq\\
    &\leq \|P_{\cS}\bfZ^{-1}_{\cS}P_{\cS}\bfZ\|_2\|(\mathbf{Id}-\bfZ^{-1}\hatZ)\|_2\|\hatZ^{-1}_{\cS}P_{\cS}\Phi_{\bfW_t}\Phi^{\top}_{\bfW_t}f\|_2\leq \\
    &\leq B\|\bfZ^{-1}_{\cS}P_{\cS}\bfZ\cP_{\cS}\|_2\|\mathbf{Id}-\bfZ^{-1/2}\hatZ\mathbf{Z}^{-1/2}\|_2\|\hatZ^{-1}_{\cS}P_{\cS}\Phi_{\bfW_t}\Phi^{\top}_{\bfW_t}\|
\end{align}
 Notice that $\|\bfZ^{-1}_{\cS}\cP_{\cS}\bfZ\cP_{\cS}\|_2=\|\bfZ^{-1}_{\cS}(\bfZ_{\cS}+\tau(\cP_{\cS}-\mathbf{Id}))\|_2\leq 2$, because $0 \preceq\cP_{\cS}\preceq 1$. In a similar fashion, based on Lemma \ref{fact1} we have $\|\hatZ^{-1}_{\cS}P_{\cS}\Phi_{\bfW_t}\Phi^{\top}_{\bfW_t}\|_2\leq\|\Phi_{\bfW_t}\Phi^{\top}_{\bfW_t}\hatZ^{-1}_{\cS}\|_2=\|(\hatZ- \tau \mathbf{Id})\hatZ^{-1}_{\cS}\|\leq \|\hatZ\hatZ^{-1}_{\cS}\|_2+1\leq \frac{8}{3}$. Utilizing Lemma \ref{Lemma:covariance_app} we have w.p. at least $1-\delta/4$:
\begin{align*}
    &|\phi^{\top}(w)\bfZ^{-1}_{\cS}(\bfZ_{\cS}-\hatZ_{\cS})\hatZ^{-1}_{\cS}P_{\cS}\Phi_{\bfW_t}\Phi^{\top}_{\bfW_t}f|\leq \frac{64B}{3}\sqrt{\max_{w\in\mathrm{supp}\varrho} \phi^{\top}(w)\bfZ^ {-1}\phi(w)\log(56T/\delta)}
\end{align*}
To bound the second term of Eq.(\ref{eq_2}),  we use standard concentration results for sub-Gaussian random variables:
\begin{align*}
    &\|\phi^{\top}(w)\bfZ^{-1}_{\cS}(\bfZ_{\cS}-\hatZ_{\cS})\hatZ^{-1}_{\cS}P_{\cS}\Phi_{\bfW_t}\|_2=\|\phi^{\top}(w)\bfZ^{-1}_{\cS}P_{\cS}(\bfZ-\hatZ)P_{\cS}\hatZ^{-1}_{\cS}P_{\cS}\Phi_{\bfW_t}\|_2\leq \\
    &\leq \frac{1}{\sqrt{\tau}}\|\cP_{\cS}\bfZ^{-1}_{\cS}\cP_{\cS}\bfZ\|_2\|(\mathbf{Id}-\bfZ^{-1}\hatZ)\|_2\|\hatZ^{-1/2}_{\cS}\cP_{\cS}\Phi_{\bfW_t}\|_2\leq\\
    &\leq \frac{4}{\sqrt{\tau}}\sqrt{\max_{w\in\mathrm{supp}\varrho} \phi^{\top}(w)\bfZ^ {-1}\phi(w)\log(56T/\delta)}\sqrt{\left\|\hatZ^{-1/2}_{\cS}\cP_{\cS}\Phi _{\bfW_t}\Phi^{\top}_{\bfW_t}\cP_{\cS}\hatZ^{-1/2}_{\cS}\right\|_2}\leq \\
    &\leq \frac{4}{\sqrt{\tau}}\sqrt{
    \max_{w\in\mathrm{supp}\varrho} \phi^{\top}(w)\bfZ^ {-1}\phi(w)\log(56T/\delta)}
\end{align*}
\begin{align*}
    \left|\phi^{\top}(w)\bfZ^{-1}_{\cS}(\bfZ_{\cS}-\hatZ_{\cS})\hatZ^{-1}_{\cS}P_{\cS}\Phi_{\bfW_t}\eta_{1:t}\right|\leq \frac{8B\sqrt{\log(4/\delta)}}{ \sqrt{\tau}}\sqrt{\max_{w\in \mathrm{supp}\varrho} \phi^{\top}(w)\bfZ^ {-1}\phi(w)\log(56T/\delta)}
\end{align*}
We now turn our attention to the second term of Eq.(\ref{eq_3}). We use the symmetric derivation as for the first term:
\begin{align}\label{eq_4}
    &|\phi^{\top}(w)\tildeZ^{-1}_{\cS}(\tildeZ_{\cS}-\bfZ_{\cS})\bfZ^{-1}_{\cS}P_{\cS}\Phi_{\bfW_t}y_t|\leq\\
    & \leq |\phi^{\top}(w)\tildeZ^{-1}_{\cS}(\tildeZ_{\cS}-\bfZ_{\cS})\bfZ^{-1}_{\cS}P_{\cS}\Phi_{\bfW_t}\Phi^{\top}_{\bfW_t} f|+|\phi^{\top}(w)\tildeZ^{-1}_{\cS}(\tildeZ_{\cS}-\bfZ_{\cS})\bfZ^{-1}_{\cS}P_{\cS}\Phi_{\bfW_t}\eta_{1:t}|
\end{align}
We again start by bounding the first term:
\begin{align}
    &|\phi^{\top}(w)\tildeZ^{-1}_{\cS}(\tildeZ_{\cS}-\bfZ_{\cS})\bfZ^{-1}_{\cS}P_{\cS}\Phi_{\bfW_t}\Phi^{\top}_{\bfW_t} f|\leq B\|\tildeZ^{-1}_{\cS}(\tildeZ_{\cS}-\bfZ_{\cS})\bfZ^{-1}_{\cS}P_{\cS}\Phi_{\bfW_t}\Phi^{\top}_{\bfW_t}\|_2=\\
    &=B\|\tildeZ_{\cS}^{-1}\cP_{\cS}\bfZ (\bfZ^{-1}\tildeZ- \mathbf{Id})\cP_{\cS}\bfZ^{-1}_{\cS}\cP_{\cS}\Phi_{\bfW_t}\Phi^{\top}_{\bfW_t}\|_2\leq \\
    &\leq B\|\tildeZ_{\cS}^{-1}\cP_{\cS}\bfZ (\bfZ^{-1}\tildeZ- \mathbf{Id})\cP_{\cS}\|_2\|\cP_{\cS}\bfZ^{-1}_{\cS}\cP_{\cS}\Phi_{\bfW_t}\Phi^{\top}_{\bfW_t}\|_2\leq \\
    &\leq B\|\tildeZ_{\cS}^{-1}\cP_{\cS}\bfZ\cP_{\cS}\|_2\|\bfZ^{-1}\tildeZ-\mathbf{Id}\|_2\|\cP_{\cS}\bfZ^{-1}_{\cS}\cP_{\cS}\Phi_{\bfW_t}\Phi^{\top}_{\bfW_t}\|_2
\end{align}

In the above derivation we used that exchanging the places of two operators keeps the spectrum the same, and the identity $\cP^2_{\cS}=\cP_{\cS}$.\\
Note that $\|\tildeZ_ {\cS}^{-1}\cP_{\cS}\bfZ\cP_{\cS}\|_2=\|\tildeZ_{\cS}^{-1}(\bfZ_{\cS}-\tau(\mathbf{Id}-\cP_{\cS}))\|_2\leq \|\tildeZ_{\cS}^{-1}\bfZ_{\cS}\|_2+1$. By using Lemmas (\ref{fact1},\ref{Lemma:covariance_app}) we have with probability at least $1-\delta/4$\footnote{The parameter $\delta$ affect the lower bound on the number of action necessary, in Lemma \ref{Lemma:covariance_app}.}, $\frac{8}{9}\bfZ\preceq\tildeZ,\Phi_{\cS}\Phi^{\top}_{\cS}+ \tau\mathbf{Id}\preceq \frac{10}{9}\bfZ$ and $\frac{72}{100}\tildeZ\preceq\frac{8}{10}\bfZ\preceq\bfZ_{\cS}\preceq \frac{10}{8}\bfZ\preceq \frac{90}{64}\tildeZ\preceq \frac{900}{512}\tildeZ_{\cS}$. Multiplying both sides with $\tildeZ^{-1/2}_{\cS}$ we have $ \tildeZ_{\cS}^{-1/2}\bfZ_{\cS}\tildeZ_{\cS}^{-1/2}\preceq \frac{900}{512}\mathbf{Id}$, and as $\tildeZ_{\cS}^{-1/2}\bfZ_{\cS}\tildeZ_{\cS}^{-1/2},\tildeZ_{\cS}^{-1}\bfZ_{\cS}$ have the same spectrum we finally have $\|\tildeZ^{-1}\cP_{\cS}\bfZ\cP_{\cS}\|_2\leq \frac{353}{128}<3$.\\

We adopt the same approach in bounding $\|\cP_{\cS}\bfZ^{-1}_{\cS}\cP_{\cS}\Phi_{\bfW_t}\Phi^{\top}_{\bfW_t}\|_2=\|\bfZ^{-1}_{\cS}\cP_{\cS}\Phi_{\bfW_t}\Phi^{\top}_{\bfW_t}\cP_{\cS}\|_2\leq \|\bfZ^{-1}_{\cS}\hatZ_{\cS}\|_2+1$. Again using Lemmas (\ref{Lemma:covariance_app},\ref{fact1}) we have $\frac{72}{125}\hatZ_{\cS}\preceq\frac{72}{100}\hatZ \preceq \frac{8}{10}\bfZ\preceq\bfZ_{\cS}\implies \|\bfZ^{-1/2}_{\cS}\hatZ_{\cS}\bfZ^{-1/2}_{\cS}\|_2\leq \frac{125}{72}<3$. Combining these inequalities we finally have w.p. at least $1- \delta/4$:
\begin{align*}
     |\phi^{\top}(w)\tildeZ^{-1}_{\cS}(\tildeZ_{\cS}-\bfZ_{\cS})\bfZ^{-1}_{\cS}P_{\cS}\Phi_{\bfW_t}\Phi^{\top}_{\bfW_t} f|< 9B\sqrt{\max_{w\in\mathrm{supp}\varrho} \phi^{\top}(w)\bfZ^ {-1}\phi(w)\log(56T/\delta)}
\end{align*}
To bound the second term of Eq.(\ref{eq_4}) we will use the $B^2$-sub-Gaussianity of the noise $\eta_{1:t}$. Note that:
\begin{align*}
&\|\phi^{\top}(w)\tildeZ^{-1}_{\cS}(\tildeZ_{\cS}-\bfZ_{\cS})\bfZ^{-1}_{\cS}P_{\cS}\Phi_{\bfW_t}\|_2\leq \|\tildeZ^{-1}_{\cS}(\tildeZ_{\cS}-\bfZ_{\cS})\bfZ^{-1}_{\cS}P_{\cS}\Phi_{\bfW_t}\|_2\leq \\
&\leq \|\tildeZ_{\cS}^{-1}\cP_{\cS}\bfZ (\bfZ^{-1}\tildeZ- \mathbf{Id})\cP_{\cS}\bfZ^{-1}_{\cS}\cP_{\cS}\Phi_{\bfW_t}\|_2\leq \frac{1}{\sqrt{\tau}}\|\tildeZ_{\cS}^{-1}\cP_{\cS}\bfZ (\bfZ^{-1}\tildeZ- \mathbf{Id})\cP_{\cS}\|_2\|\bfZ^{-1/2}_{\cS}\cP_{\cS}\Phi_{\bfW_t}\|_2\leq \\
&\leq \frac{1}{\sqrt{\tau}}\|\cP_{\cS}\tildeZ_{\cS}^{-1}\cP_{\cS}\bfZ\|_2\| (\bfZ^{-1}\tildeZ- \mathbf{Id})\|_2\|\bfZ^{-1/2}_{\cS}\cP_{\cS}\Phi_{\bfW_t}\|_2=\\
&= \frac{1}{\sqrt{\tau}}\|\tildeZ_{\cS}^{-1}\cP_{\cS}\bfZ\cP_{\cS}\|_2\| (\bfZ^{-1/2}\tildeZ\bfZ^{-1/2}- \mathbf{Id})\|_2\sqrt{\|\bfZ^{-1/2}_{\cS}\cP_{\cS}\Phi_{\bfW_t}\Phi^{\top}_{\bfW_t}\cP_{\cS}\bfZ_{\cS}^{-1/2}\|_2}\leq \\
&\leq \frac{1}{\sqrt{\tau}}\frac{353}{6 4}\sqrt{\max_{w\in\mathrm{supp}\varrho} \phi^{\top}(w)\bfZ^ {-1}\phi(w)\log(56T/\delta)}\sqrt{\|\bfZ^{-1/2}_{\cS}\hatZ_{\cS}\bfZ^{-1/2}_{\cS}\|_2}\leq \\
&< \frac{6}{\sqrt{\tau}}\sqrt{2\max_{w\in\mathrm{supp}\varrho} \phi^{\top}(w)\bfZ^ {-1}\phi(w)\log(56T/\delta)}
\end{align*}
Now by sub-Gaussianity we have w.p at least $1-\delta/4$:
\begin{align*}
    |\phi^{\top}(w)\tildeZ^{-1}_{\cS}(\tildeZ_{\cS}-\bfZ_{\cS})\bfZ^{-1}_{\cS}P_{\cS}\Phi_{\bfW_t}\eta_{1:t}|\leq \frac{12B\sqrt{\log(4/\delta)}}{\sqrt{\tau}}\sqrt{2\max_{w\in\mathrm{supp}\varrho} \phi^{\top}(w)\bfZ^ {-1}\phi(w)\log(56T/\delta)}
\end{align*}
To express the final error bound in terms of the variance with respect to the projected operator $\widetilde\bfZ^{-1}$ we use Lemmas(\ref{fact1}, \ref{Lemma:covariance_app}): $\bfZ^ {-1}\preceq\frac{10}{9}\widetilde\bfZ^ {-1}\preceq \frac{50}{27} \widetilde\bfZ_{\cS}^ {-1}$. Finally combining all the bounds in this step we can write w.p. at least $1-\delta$:
\begin{align*}
    E_2< \left(90B\log\left(\frac{56T}{\delta}\right)+\frac{52B\log\left(\frac{4}{\delta}\right)}{\sqrt{\tau}}\right)\widetilde\sigma(w)&\leq \max_{w\in \cW}\left(90B\log\left(\frac{56T}{\delta}\right)+\frac{52B\log\left(\frac{56T}{\delta}\right)}{\sqrt{\tau}}\right)\widetilde\sigma(w)=\\
    &=\left(90B\sqrt{\log\left(\frac{56T}{\delta}\right)}+\frac{52B\sqrt{\log\left(\frac{56T}{\delta}\right)\log\left(\frac{4}{\delta}\right)}}{\sqrt{\tau}}\right)\widetilde\sigma_{\max}
\end{align*}

\textbf{Step 3: Bounding $E_3$.}\\

We  note that based on Lemma \ref{Lemma_kernel_trick} we have:
\begin{align*}
    |\widehat{\underline{\mu}}_t(w)-\widehat\mu_t(w)|= \left|\bfk^{\top}_{\cS_t}(w)\left(\bfK_{\cS_t,\cR_t}\bfK_{\cR_t,\cS_t}+\tau\bfK_{\cS_t,\cS_t}\right)^{-1/2} \bfZ_{\mathrm{priv}}\right|
\end{align*}
$\mathbf{Z}_{\mathrm{priv}}$ is a noise vector, drawn independently of $\bfk^{\top}_{\cS_t}(\cdot)\left(\bfK_{\cS_t,\cR_t}\bfK_{\cR_t,\cS_t}+\tau\bfK_{\cS_t,\cS_t}\right)^{-1/2}$. To bound the dot product, we will use Chernoff-Hoefding along with the sub-Gaussianity of $\mathbf{Z}_{\mathrm{priv}}$. By definition, $\bfZ_{\mathrm{priv}}$ is $\sigma_0^2$-sub Gaussian, thus to bound the dot product it suffices to bound:
\begin{align}\label{Eq:sensitivity_bounding}
    &\|\bfk^{\top}_{\cS_t}(w)\left(\bfK_{\cS_t,\cR_t}\bfK_{\cR_t,\cS_t}+\tau\bfK_{\cS_t,\cS_t}\right)^{-1/2}\|_2\leq\\
    &\leq\sqrt{\bfk^{\top}_{\cS_t}(w)\left(\bfK_{\cS_t,\cR_t}\bfK_{\cR_t,\cS_t}+\tau\bfK_{\cS_t,\cS_t}\right)^{-1}\bfk^{\top}_{\cS_t}}\leq\\
    &\leq \sqrt{\bfk^{\top}_{\cS_t}(w)\left(\bfK_{\cS_t,\cR_t}\bfK_{\cR_t,\cS_t}+\tau\bfK_{\cS_t,\cS_t}\right)^{-1}\bfk_{\cS_t}(w)}=\\
    &=\sqrt{\bfk^{\top}_{\cS_t}(w)\bfK^{-1/2}_{\cS_t,\cS_t}\left(\bfK^{-1/2}_{\cS_t,\cS_t}\bfK_{\cS_t\cR_t}\bfK_{\cR_t,\cS_t}\bfK^{-1/2}_{\cS_t,\cS_t}+\tau\mathbf{I}_{\cS_t}\right)^{-1}\bfK^{-1/2}_{\cS_t,\cS_t}\bfk_{\cS}(w)}=\\
    &=\sqrt{\phi^{\top}(w)\Phi_{\cS_t}\bfK^{-1/2}_{\cS_t,\cS_t}\left(\bfK^{-1/2}_{\cS_t,\cS_t}\Phi^{\top}_{\cS_t}\Phi_{\cR_r}\Phi^{\top}_{\cR_r}\Phi_{\cS_t}\bfK^{-1/2}_{\cS_t,\cS_t}+\tau\mathbf{I}_{\cS_t}\right)^{-1}\bfK^{-1/2}_{\cS_t,\cS_t}\Phi^{\top}_{\cS_t}\phi(w)}=\\
    &=\sqrt{\phi^{\top}(w)\cP_{\cS_t}\left(\Phi_{\cR_r}\Phi^{\top}_{\cR_r}\cP_{\cS_t}+\tau\mathbf{I}_{\cS_t}\right)^{-1}\phi(w)}=\\
    &=\sqrt{\phi^{\top}(w)\cP_{\cS_t}\left(\Phi_{\cR_r}\Phi^{\top}_{\cR_r}\cP^ 2_{\cS_t}+\tau\mathbf{I}_{\cS_t}\right)^{-1}\phi(w)}=\\
    &=\sqrt{\phi^{\top}(w)\left(\cP_{\cS_t}\Phi_{\cR_r}\Phi^{\top}_{\cR_r}\cP_{\cS_t}+\tau\mathbf{I}_{\cS_t}\right)^{-1}\cP_{\cS_t}\phi(w)}
\end{align}

We next show that $(\cP_{\cS_t}\Phi_{\cR_t}\Phi^{\top}_{\cR_t}\cP_{\cS_t}+\tau\mathbf{I}_{\cS_t})^{-1},\cP_{\cS_t}$ commute. The proof once again relies on the Lemma \ref{kernel_trick_mat_identity}:
\begin{align*}
    \cP_{\cS_t}(\cP_{\cS_t}\Phi_{\cR_t}\Phi^{\top}_{\cR_t}\cP_{\cS_t}+\tau\mathbf{I}_{\cS_t})^{-1}&=\cP_{\cS_t}(\cP_{\cS_t}\Phi_{\cR_t}\Phi^{\top}_{\cR_t}\cP^2_{\cS_t}+\tau\mathbf{I}_{\cS_t})^{-1}=\\
    &=(\cP^2_{\cS_t}\Phi_{\cR_t}\Phi^{\top}_{\cR_t}\cP_{\cS_t}+\tau\mathbf{I}_{\cS_t})^{-1}\cP_{\cS}=(\cP_{\cS_t}\Phi_{\cR_t}\Phi^{\top}_{\cR_t}\cP_{\cS_t}+\tau\mathbf{I}_{\cS_t})^{-1}\cP_{\cS}
\end{align*}

As $\bfI- \cP_{\cS_t}\succeq 0,(\cP_{\cS_t}\Phi_{\cR_t}\Phi_{\cR_t}\cP_{\cS_t}+\tau\mathbf{Id})\succeq 0$ from commutativity we have $(\cP_{\cS_t}\Phi_{\cS_t}\Phi_{\cS_t}\cP_{\cS_t}+\tau\mathbf{Id})^{-1}(\bfI- \cP_{\cS_t})\succeq 0 $ and thus $\phi^{\top}(w)(\cP_{\cS_t}\Phi_{\cR_t}\Phi_{\cR_t}\cP_{\cS_t}+\tau\mathbf{I}_{\cS_t})^{-1}\phi(w)\leq \phi^{\top}(w)(\cP_{\cS_t}\Phi_{\cR_t}\Phi_{\cR_t}\cP_{\cS_t}+\tau\mathbf{Id})^{-1}\cP_{\cS_t}\phi(w)=\widetilde{\sigma}^2(w)$\\

We now apply Chernof-Hoefding inequality on the dot product. The following bound holds with probability at least $1-\delta$:
\begin{align*}
    E_3\leq 2\log\left(\frac{1}{\delta}\right)\widetilde\sigma_0\widetilde\sigma(w)\leq  \frac{8B\log{T}}{\varepsilon_{\textsc{DP}}}\log\left(\frac{1}{\delta}\right)\sqrt{\log\left(\frac{1.25\log{T}}{\delta_{\textsc{DP}}}\right)}\widetilde\sigma^2_{\mathrm{max}}
\end{align*}

\textbf{Step 4: Total approximation error bound .}\\

We can now finally add the errors $E_1,E_2,E_3$ and scale the error probability by $3$ to obtain w.p. at least $1-\delta$. Using  Lemmas(\ref{fact1}, \ref{Lemma:covariance_app}) once more $\left(\bfZ^ {-1}\preceq\frac{10}{9}\widetilde\bfZ^ {-1}\preceq \frac{50}{27} \widetilde\bfZ_{\cS}^ {-1}\right)$ we have:
\begin{align*}
&|f(w)-\widehat{\underline{\mu}}_t(w)|\leq E_1+E_2+E_3\leq\\
&\leq \widetilde\sigma_{\mathrm{max}}\underbrace{\left(90B\sqrt{\log\left(\frac{168T}{\delta}\right)}+\frac{52B\sqrt{\log\left(\frac{168T}{\delta}\right)\log\left(\frac{12}{\delta}\right)}}{\sqrt{\tau}}+3B\sqrt{2\log\left(\frac{6}{\delta}\right)}+ \sqrt{24\tau}\right)}_{\beta(\delta)}+\\
&+\widetilde\sigma^2_{\mathrm{max}}\underbrace{\frac{8B\log{T}}{\varepsilon_{\textsc{DP}}}\log\left(\frac{3}{\delta}\right)\sqrt{\log\left(\frac{1.25\log{T}}{\delta_{\textsc{DP}}}\right)}}_{\beta_1(\varepsilon_{\mathrm{DP}},\delta_{\mathrm{DP}},\delta)}
\end{align*}
To further enforce the approximation error holds uniformly over $\mathrm{supp}\varrho$ we use the union bound over the $\mathrm{supp}\varrho\subset\cW$
\begin{align*}
&\max_{w\in \mathrm{supp}\varrho}|f(w)-\widehat{\underline{\mu}}_t(w)|\leq\beta\left(\frac{\delta}{|\cW|}\right)\widetilde\sigma_{\mathrm{max}}+\beta_1\left(\varepsilon_{\mathrm{DP}},\delta_{\mathrm{DP}},\frac{\delta}{|\cW|}\right)\widetilde\sigma^2_{\mathrm{max}}
\end{align*}
The only difference in the approximation error for the \text{LDP} estimator is that variance of the privatization noise, $\bfZ_{\mathrm{priv}, \mathrm{LDP}}$ is $\sqrt{t}\sigma_0$. Thus, $E_1,E_2$ remain unchanged and the upper bound for $E_3$ is scaled up by $\sqrt{t}$:
\begin{align*}
&\max_{w\in \mathrm{supp}\varrho}|f(w)-\widehat{\underline{\mu}}_{t,\text{LDP}}(w)|\leq\beta\left(\frac{\delta}{|\cW|}\right)\widetilde\sigma_{\mathrm{max}}+\beta_{1,\mathrm{LDP}}\left(\varepsilon_{\mathrm{DP}},\delta_{\mathrm{DP}},\frac{\delta}{|\cW|}\right)\widetilde\sigma^2_{\mathrm{max}}
\end{align*}
where $\beta_{1,\mathrm{LDP}}=\sqrt{t}\beta_{1}$.\\

\textbf{Step 5: JDP and LDP privacy guarantees.}\\

To show that $\widehat{\underline{\mu}}_t$ is $(\varepsilon_{\mathrm{DP}},\delta_{\mathrm{DP}})$ with respect to the the data-set $\cD_t$ it suffices to show that 

$$g:=\left(\bfK_{\cS_t,\cR_t}\bfK_{\cR_t,\cS_t}+\tau\bfK_{\cS_t,\cS_t}\right)^{-1/2}\bfK_{\cS_r,\bfW_t}\bfy_t+\bfZ_{\mathrm{priv}}$$

is $(\varepsilon_{\mathrm{DP}},\delta_{\mathrm{DP}})$ wrt $\cD_t$ as the approximating and the projection set $\cS_t,\cR_t$ are data-invariant. To  this end, we use Gaussian Mechanism \citep{Dwork}. It suffices to bound that the $\ell_2$ sensitivity of $h(\cD_t):=(\bfK_{\cS_t,\cR_t}\bfK_{\cR_t,\cS_t}+\tau\bfK_{\cS_t,\cS_t})^{-1/2}\bfK_{\cS_t,\bfW_t}\bfy_t$:
\begin{align*}
    \Delta_2h=\sup_{\cD_t, \cD'_t}\|h(\cD_t)-h(\cD_t)\|_2&\leq 2\sup_{w\in \mathrm{supp}\varrho, y_t\in [-B,B]} \|(\bfK_{\cS_t,\cR_t}\bfK_{\cR_t,\cS_t}+\tau\bfK_{\cS_t,\cS_t})^{-1/2}\bfk_{\cS_t}(w)y_t\|_2\leq\\
    &\leq 2B \sup_{w\in \mathrm{supp}\varrho}\|(\bfK_{\cS_t,\cR_t}\bfK_{\cR_t,\cS_t}+\tau\bfK_{\cS_t,\cS_t})^{-1/2}\bfk_{\cS_t}(w)\|_2
\end{align*}
However note that this exactly the expression we bounded in Eq.(\ref{Eq:sensitivity_bounding}) and thus $\Delta_2h\leq 2B\widetilde\sigma_{\mathrm{max}}$. By Gaussian Mechanism,  adding noise distributed as $\bfZ_{\mathrm{priv}}\sim\cN\left(0,\left( \widetilde\sigma_{\max}\frac{4B\log{T}}{\varepsilon_{\textsc{DP}}} \cdot \sqrt{\log\left(\frac{1.25\log{T}}{\delta_{\textsc{DP}}}\right)}\right)^2\mathbf{I}_t\right)$ indeed guarantees that $g$ is private wrt $\cD_t$ and thus by post-processing property of differential privacy so is
$\widehat{\underline{\mu}}_t$.\\

$(\varepsilon_{\mathrm{DP}},\delta_{\mathrm{DP}})$-LDP guarantee of $\widehat{\underline{\mu}}_t$ follows in a similar fashion. We have to show that:
\begin{align*}
g_{\mathrm{LDP}}:=\left(\bfK_{\cS_t,\cR_t}\bfK_{\cR_t,\cS_t}+\tau\bfK_{\cS_t,\cS_t}\right)^{-1/2}\bfk_{\cS_r}(w_t)y_t+\bfZ_{\mathrm{priv}}
\end{align*}
is $(\varepsilon_{\mathrm{DP}},\delta_{\mathrm{DP}})$-private with respect to the single data-point $w_t=(c_t,x_t)$. It suffices to bound the $\ell_2$ sensitivity of $h_{\mathrm{LDP}}(w_t):=(\bfK_{\cS_t,\cR_t}\bfK_{\cR_t,\cS_t}+\tau\bfK_{\cS_t,\cS_t})^{-1/2}\bfk_{\cS_t}(w_t)y_t$. By the exactly the same argument as in the JDP setting we have $\Delta_2h_{\mathrm{LDP}}\leq 2B \widetilde\sigma_{\mathrm{max}}$. Utilizing Gaussian Mechanism we have that $g_{\mathrm{LDP}}$ is $(\varepsilon_{\mathrm{DP}},\delta_{\mathrm{DP}})$-differentially private with respect to the data-point $(c_t,x_t)$. In thus follows that $\widehat{\underline{\mu}}_{t,\mathrm{LDP}}$ is $(\varepsilon_{\mathrm{LDP}},\delta_{\mathrm{LDP}})$.
\end{proof}

We now state and prove the main result upper bounding the cumulative regret of our algorithm.

\begin{theorem}\label{theorem:main_theorem_Appendix}
    Consider the problem of contextual kernelized bandits described in the Section~\ref{sec:problem_formulation} where Assumptions~\ref{assumption:bounded_eigen_functions}-\ref{assumption:contexts} hold. If \textsc{CAPRI} is run for a time horizon of $T$ with $T >T_0(\delta):=\max\left(\overline{T},\left(\frac{180\log(14T/\delta)}{F^2}\right)^{\beta_p/\beta_p-1}\right)$, and privacy parameters $\varepsilon_{\DP} > 0$ and $\delta_{\DP} \in (0,1)$ under JDP (resp. LDP), then 
    \begin{itemize}
        \item \textsc{CAPRI} satisfies $(\varepsilon_{\DP}, \delta_{\DP})$-JDP (resp. LDP)
        \item The regret incurred by the algorithm satisfies:
        \begin{align*}
        &\textsc{R}_{T, \mathrm{JDP}}(\sA)\leq\sqrt{T} \left(4B+ \frac{135\log{T}F^2}{26}\beta\left(\frac{\delta}{|\cW|T\log{T}}\right)\sqrt{2\gamma_T}\right)+ \frac{135\log{T}F^2}{13}\beta_1\left(\varepsilon_{\textsc{DP}},\delta_{\textsc{DP}},\frac{\delta}{T\log T|\cW|}\right)\gamma_T\\
        &\textsc{R}_{T,\mathrm{LDP}}(\sA)\leq \sqrt{T}\left( 4B \frac{135\log{T}F^2}{26}\beta\left(\frac{\delta}{|\cW|T\log{T}}\right)\sqrt{2\gamma_T}+ \frac{135\log{T}F^2}{13}\beta_1\left(\varepsilon_{\textsc{DP}},\delta_{\textsc{DP}},\frac{\delta}{T\log T|\cW|}\right)\gamma_T\right)
        \end{align*}    
    \end{itemize}
     Regret bound holds with probability $1- \delta_{\textsc{ERR}}$ over the randomness in the algorithm and the noise sequence. Here $\overline{T}$ is a constant that depends only on the kernel $k$ and the context distribution,  and $\beta, \beta_{1}$ are confidence parameters introduced in Lemma \ref{Lemma:Appendix_estimator} :
     \begin{align*} &\beta(\delta):=\left(90B\sqrt{\log\left(\frac{168T}{\delta}\right)}+\frac{52B\sqrt{\log\left(\frac{168T}{\delta}\right)\log\left(\frac{12}{\delta}\right)}}{\sqrt{\tau}}+3B\sqrt{2\log\left(\frac{6}{\delta}\right)}+ \sqrt{24\tau}\right),\\
     &\beta_1(\varepsilon_{\mathrm{DP}},\delta_{\mathrm{DP}},\delta):=\frac{8B\log{T}}{\varepsilon_{\textsc{DP}}}\log\left(\frac{3}{\delta}\right)\sqrt{\log\left(\frac{1.25\log{T}}{\delta_{\textsc{DP}}}\right)}
     \end{align*}
\end{theorem}
\begin{proof}
    We will proceed in several steps. We will first bound the simple regret in the $r^{\text{th}}$ epoch as a function of maximum projected variance (see Lemma \ref{Lemma:Appendix_estimator}) in $(r-1)^{\text{th}}$ epoch. We will then establish a bound on the maximum projected variance in terms of maximal information gain, and then finally bound the cumulative regret.\\
    In terms of privacy guarantees, the proof will largely follow from the privacy guarantees in Lemma \ref{Lemma:Appendix_estimator} along with invariance to post-processing with respect to the subsequent epochs.\\

    \textbf{Step 1: Bounding the simple regret in $r^{\text{th}}$ epoch.}\\
    
    We define the simple regret in the $r$-th epoch as:
    \begin{align*}
     \overline\Delta_r=\max_{(c,x)\in\bfW_r}\max_{x_c\in \cX}f(c,x_c)-f(c,x)
    \end{align*}
    In other words, it is defined as the largest error between the reward function at the maximum point and the point played by the agent, taken over all context-point pairs in the $r^{\text{th}}$ epoch.\\

    We bound the simple regret in the first epoch  as $\overline\Delta_1\leq \max_{w\in \cW}f(w)-\min_{w\in \cW}f(w)\leq 2B$.\\

    For $r\geq 2$, w.p. at least $1-\delta/T$ all over epochs we will show a bound $\overline{\Delta}_{r}\leq 6\Delta_{r-1}$ where $\Delta_{r} =\beta\left(\frac{\delta}{|\cW|T\log{T}}\right)\widetilde\sigma_{r,\max}+\beta_1\left(\varepsilon_{\textsc{DP}},\delta_{\textsc{DP}},\frac{\delta}{|\cW|T\log{T}}\right)\widetilde\sigma^2_{r,\max}$\footnote{The expression for the LDP confidence bound is the same with $\beta_1$ being replaced with $\beta_{1, \mathrm{LDP}}$} is the confidence parameter used in \textsc{CAPRI}.\\

    We will show that $ \max_{x\in \cX}f(c_t,x)-f(c_t,x_t)\leq 6\Delta_{r-1}$ holds over all time instances in the $r$-th epoch(in probability). In particular, we will show that this is the case for all elements of the set $\cX_r(c_t):=\{x\in \cX_{r-1}(c_t)|\underline{\widehat{\mu}}_{r-1}(c_t,x)\geq \max_{x\in \cX_{r-1}(c_t)}\underline{\widehat{\mu}}_{r-1}(c_t,x)-4\Delta_{r-1}\}$. Suppose that, by the way of contradiction, for a $z\in \cX_{r}(c_t)$ we have $f(c_t,x^{*}_t)-f(c_t,z)\geq 6\Delta_{r-1}$ where $x^{*}_t=\mathrm{argmax}_{x\in \cX}f(c_t,x)$ and let $z^*= \mathrm{argmax}_{z\in \cX_{r-1}(c_t)}\underline{\widehat{\mu}}_{r-1}(c_t,z)$. \\

    Clearly $z^{*}\in \cX_{r}(c_t)$, we will show that $x^*_t\in \cX_{r}(c_t)$. We will show this by inducting on $r$. For $r=1$, $\cX_{1}(c_t)\equiv \cX$ thus the claim is clearly true. Suppose  $x_t^{*}\in \cX_{r-1}(c_t)$ we will show $x^{*}_t\in \cX_{r}(c_t)$. Note that  by the result of Lemma \ref{Lemma:Appendix_estimator} with probability at least $1-\frac{\delta}{T\log T}$ , $\underline{\widehat{\mu}}_{r-1}(c_t,z^{*})-\underline{\widehat{\mu}}_{r-1}(c_t,x^{*}_t)\leq |f(c_t,z^{*})-\underline{\widehat{\mu}}_{r-1}(c_t,z^{*})|+|f(c_t,x^{*}_t)-\underline{\widehat{\mu}}_{r-1}(c_t,x_t^{*})|+(f(c_t,z^{*})-f(c_t,x^{*}_t))\leq 2\Delta_{r-1}$ and thus indeed  $x^{*}_t\in \cX_{r}(c_t)$.\\
    We can now write:
    \begin{align*}
        f(c_t,x^{*}_t)-f(c_t,z)&\leq | f(c_t,x^{*}_t)-\underline{\widehat{\mu}}_{r-1}(c_t,x^{*}_t)|+|f(c_t,z)-\underline{\widehat{\mu}}_{r-1}(c_t,z)|+|\underline{\widehat{\mu}}_{r-1}(c_t,x^{*}_t)-\underline{\widehat{\mu}}_{r-1}(c_t,z)|\leq\\
        &\leq 2\Delta_r+ \max_{x\in \cX_{r-1}(c_t)}\underline{\widehat{\mu}}_{r-1}(c_t,x)- \underline{\widehat{\mu}}_{r-1}(c_t,z)\leq 6\Delta_{r-1}
    \end{align*}
     contradicting the original assumption and proving the claim.\\

     Taking the union bound over all epochs we have with probability least $1-\delta/T$: 
     \begin{align*}
    &\overline{\Delta}_i\leq \begin{cases}
                2B & i=1\\
                6\Delta_{i-1} & i>1
            \end{cases}            
    &\overline{\Delta}_{i,\textsc{LDP}}\leq \begin{cases}
                2B & i=1\\
                6\Delta_{i-1,\textsc{LDP}} & i>1
    \end{cases}
    \end{align*}

\textbf{Step 2: Bounding $\widetilde\sigma^2_{r,\mathrm{max}}$.}
\\

In this section $\bfZ$ reffers to the expected covariance matrix of the sample measure $\varrho_r$, in the $r^{\text{th}}$ epoch. For notational convenience, we drop the subscript $r$.\\
Note that by Lemma \ref{Lemma:covariance_app} w.p. $1- \delta$ \footnote{This dependence on the probability  is captured in $T\geq T_0(\delta)$} we have  $\frac{8}{9}\bfZ \preceq\Phi_{\cS_r}\Phi^{\top}_{\cS_r}+\tau\mathbf{I}\preceq\frac{10}{9}\bfZ$ where $\bfZ$ is the expected covariance estimator(for more details please see Lemma \ref{Lemma:covariance_app}). Using the before mentioned lemma also gives us $\frac{8}{9}\bfZ \preceq\Phi_{\cR_r}\Phi^{\top}_{\cR_r}+\tau\mathbf{I}\preceq\frac{10}{9}\bfZ$. We can now write:
\begin{align*}
  \frac{8}{10}\left(\Phi_{\cR_r}\Phi^{\top}_{\cR_r}+\tau\mathbf{Id}\right)\preceq\frac{8}{9}\bfZ\preceq\Phi_{\cS_r}\Phi^{\top}_{\cS_r}+\tau\mathbf{I}\preceq \frac{10}{9}\bfZ\preceq \frac{10}{8}\left(\Phi_{\cR_r}\Phi^{\top}_{\cR_r}+\tau\mathbf{Id}\right) 
\end{align*}
By Lemma \ref{fact1} we have:
\begin{align}\label{equation_variance_bound}
\frac{3}{5}\left(\Phi_{\cR_r}\Phi^{\top}_{\cR_r}+\tau\mathbf{Id}\right)\preceq \cP_{\cS}\Phi_{\cR_r}\Phi^{\top}_{\cR_r}\cP_{\cS_r}+\tau\mathbf{I}\preceq \frac{5}{3}\left(\Phi_{\cR_r}\Phi^{\top}_{\cR_ r}+\tau\mathbf{Id}\right)
\end{align}
Recall that $\widetilde{\sigma}^2_{\mathrm{max},r}=\max_{w\in \mathrm{supp}\varrho}\phi^{\top}(w)\left(\cP_{\cS_r}\Phi_{\cR_r}\Phi^{\top}_{\cR_r}\cP_{\cS_r}+\tau\mathbf{I}\right)^{-1}\phi(w)$. Using Eq(\ref{equation_variance_bound}) we have:
\begin{align*}
    \widetilde{\sigma}^2_{\mathrm{max},r}\leq \frac{5}{3} \max_{w\in \mathrm{supp}\varrho}\phi^{\top}(w)\left(\Phi_{\cR_r}\Phi^{\top}_{\cR_r}+\tau\mathbf{Id}\right)^{-1}\phi(w)\leq \frac{5}{24} \max_{w\in \mathrm{supp}\varrho}\phi^{\top}(w)\bfZ^{-1}\phi(w)\leq \frac{90F^2}{104}\frac{\gamma_T}{T}
\end{align*}

, where the last inequality is the direct consequence of Lemma.\ref{variance_bounding}.\\

\textbf{Step 3: Finalizing the cumulative regret bound.}\\

We can re-write the regret as:

\begin{align*}
    R_T(\sA)=&\mathbb{E}_{c_1,c_2,\dots c_T\sim \kappa}\left[\sum_{t=1}^{T}\max_{x\in \cX}f(c_t,x)-f(c_t,x_t)\right]= \sum_{t=1}^{T}\mathbb{E}_{c_1,c_2,\dots c_{t}}\left[\max_{x\in \cX}f(c_t,x)-f(c_t,x_t)\right]
\end{align*}
For a fixed epoch, the bound  from Step 1 holds w.p. at least $1-\frac{\delta}{T\log T}$ over the previously drawn contexts. Conditioning on this event, we can bound the expected simple regret w.p. at least $1-\frac{\delta}{\log{T}}$ over the reward sequence:
\begin{align}\label{eq_proof}
    \text{$t$ is in $r$-th epoch, $r>1$ , }\mathbb{E}_{\{c_i\}_{i=1}^{t}\sim\kappa}\left[\max_{x\in \cX}f(c_t,x)-f(c_t,x_t)\right]&\leq 6\Delta_{r-1}+\frac{2B\delta}{T\log T }
\end{align}
Summing Eq.(\ref{eq_proof}) over all time instance(and thus over all epochs) we have with probability at least $1-\delta$:
\begin{align*}
        &R_{T,\text{JDP}}(\sA)\leq 2B\sqrt{T}+\frac{2B\delta}{\log T}+\sum_{r=2}^{\log{T}}6\Delta_{r-1}T_r\\
        &\leq 4B\sqrt{T}+\sum_{r=2}^{\log T}6T_r\left(\beta\left(\frac{\delta}{|\cW|T\log{T}}\right)\widetilde\sigma_{r-1,\max}+\beta_1\left(\varepsilon_{\textsc{DP}},\delta_{\textsc{DP}},\frac{\delta}{|\cW|T\log{T}}\right)\widetilde\sigma^2_{r-1,\max}\right)\leq\\
        &\leq 4B\sqrt{T}+\sum_{r=2}^{\log T} \frac{135F^2}{26}\beta\left(\frac{\delta}{|\cW|T\log{T}}\right)\sqrt{2\gamma_TT_{r-1}}+ \frac{135F^2}{13}\beta_1\left(\varepsilon_{\textsc{DP}},\delta_{\textsc{DP}},\frac{\delta}{T\log T|\cW|}\right)\gamma_T\leq\\
        &\leq  4B\sqrt{T}+ \frac{135\log{T}F^2}{26}\beta\left(\frac{\delta}{|\cW|T\log{T}}\right)\sqrt{2\gamma_TT}+ \frac{135\log{T}F^2}{13}\beta_1\left(\varepsilon_{\textsc{DP}},\delta_{\textsc{DP}},\frac{\delta}{T\log T|\cW|}\right)\gamma_T
\end{align*}

Using the analogous result of Step 1 for the LDP setting we have with probability at least $1-\delta$:
\begin{align*}
    &R_{T,\text{JDP}}(\sA)\leq 4B\sqrt{T}+ \frac{135\log{T}F^2}{26}\beta\left(\frac{\delta}{|\cW|T\log{T}}\right)\sqrt{2\gamma_TT}+ \frac{135\log{T}F^2}{13}\beta_{1,\text{LDP}}\left(\varepsilon_{\textsc{DP}},\delta_{\textsc{DP}},\frac{\delta}{T\log T|\cW|}\right)\gamma_T\leq \nonumber \\
    & \leq 4B\sqrt{T}+ \frac{135\log{T}F^2}{26}\beta\left(\frac{\delta}{|\cW|T\log{T}}\right)\sqrt{2\gamma_TT}+ \frac{135\log{T}F^2}{13}\beta_1\left(\varepsilon_{\textsc{DP}},\delta_{\textsc{DP}},\frac{\delta}{T\log T|\cW|}\right)\gamma_T\sqrt{T}
\end{align*}

\textbf{Step 4: Privacy Guarantees.}\\

\textbf{JDP guarantee.}\\

In \textit{Rand-Proj-USCA}, the points in the $r$-th epoch are are queried uniformly from the prunned domain set . It is thus clear that the distribution of the points in $r$-th epoch \textit{does not} depend on the (context, reward) pairs in the same epoch. Hence, it is sufficient to show the points in the $r$-th epoch are differentially  private with respect to the data-set accumulated in the previous $r-1$ epochs, $\cup_{i=1}^{r-1}\{\bfC_i, \bfY_i\}$ \footnote{Here we define $\bfC_i=\{(x, c)\in \bfW_i|c\}$ as the context set in the $i^{\text{th}}$ epoch.}
As we previously concluded,  the points in $r$-th epoch only depend on the dataset $\cup_{i=1}^{r-1}\{\bfC_i, \bfY_i\}$ through sampling of the domain, which is in turn a random function of the privatized estimators $\{\underline{\widehat{\mu}}_{i}\}_{i=1}^{r-1}$. We will first prove an auxiliary claim :

\begin{claim}
    Privatized estimator $\underline{\widehat{\mu}}_{i}$, calculated in the $i$-th epoch is $(i\varepsilon_{\textsc{DP}} /\log T, i\delta_{\textsc{DP}}/\log T)$-DP with respect to the dataset $\{\bfC_j, \bfY_j\}^{i-1}_{j=1}$
\end{claim}
\begin{proof}
    We proceed by induction on $i$. The case for $i=1$ follows immediately from the privacy guarantee shown in  Lemma \ref{Lemma:Appendix_estimator}.
    Now suppose the hypothesis holds true for $i-1$.  $\underline{\widehat{\mu}}_{i}$ can be seen as a post-processing of the composition of two private mappings. The first one is $\{\bfW_{j}, \bfC_{j}\}_{j=1}^{i-1}\rightarrow \{\bfW_i,\bfY_i\} $, by induction hypothesis and post-processing invariance ( $ \{\bfW_i,\bfY_i\}$ only depends on $\{\underline{\widehat{\mu}}_j\}_{j=1}^{i-1}$) it is $((i-1)\varepsilon_{\textsc{DP}} /\log T, (i-1)\delta_{\textsc{DP}}/\log T)$-DP. The second mapping is $ \{\bfW_i,\bfY_i\}\rightarrow \underline{\widehat{\mu}}_i$. More specifically $\widehat{\mu}_i$ is a function of queried-points in the $r$-th epoch, and by  the result of Lemma \ref{Lemma:Appendix_estimator}$(\varepsilon_{\textsc{DP}} /\log T, \delta_{\textsc{DP}}/\log T)$-DP. The result of the claim now follows from  composition theorem of \cite{Dwork}.
\end{proof}
As shown in the claim the set $\{\underline{\widehat{\mu}}_i\}_{i=1}^{r-1}$ is $((r-1)\varepsilon_{\textsc{DP}} /\log T, (r-1)\delta_{\textsc{DP}}/\log T)$-DP wrt $\{\bfC_j, \bfY_j\}_{j=1}^{r-1}$. The query points in the $r$-th epoch, are a result of post-processing from  $\{\underline{\widehat{\mu}}_i\}_{i=1}^{r-1}$(they are uniformly sampled from the prunned domain) and thus the desired conclusion follows.\\

\textbf{LDP  guarantee.}\\

At time instant $t$ , in epoch $r$ only the statistic $g_{\mathrm{LDP}}((c_t,x_t),y_t)):=y_t(\bfK_{\cS_r,\cR_r}\bfK_{\cR_r,\cS_r}+\tau\bfK_{\cS_r,\cS_r})^{-1}\bfk_{\cS_r}(c_t,y_t)+\bfZ_{\text{priv}}$ is uploaded to the algorithm. As shown in Lemma \ref{Lemma:Appendix_estimator}(Step 5), $g_{\text{LDP}}$ is $(\varepsilon_{\text{DP}},\delta_{\text{DP}})$differentially private with respect to the context-reward pair $(c_t,y_t)$. This holds over all time instances and all epochs , hence the algorithm is indeed locally differentially private.

\end{proof}

\newpage

\section{Appendix B. Kernel Trick}\label{kernel_trick}

In this section we show how the estimator $\widehat\mu_T$ introduced in Lemma \ref{Lemma:Appendix_estimator} can be efficiently calculated. 
\begin{lemma} \label{Lemma_kernel_trick}
The projected estimator and variance:
\begin{align*}
    &\widehat\mu_T(w):=\phi^{\top}(w)(\cP_{\cS}\Phi_{\cR}\Phi^{\top}_{\cR}\cP_{\cS}+\tau \mathbf{Id})^{-1}\cP_{\cS}\Phi_{\bfW_T}\bfy_T\\
    &\widetilde\sigma_T^2(w):=\phi^{\top}(w)(\cP_{\cS}\Phi_{\cR}\Phi^{\top}_{\cR}\cP_{\cS}+\tau\mathbf{Id})^{-1}\Phi_{\bfW_T}
\end{align*} 
introduced in Lemma (\ref{Lemma:Appendix_estimator}) can be equivalently written as:
\begin{align*}
    &\widehat\mu_T(w)=\bfk_{\cS}(w)(\bfK_{\cS,\cR}\bfK_{\cR,\cS}+\tau\bfK_{\cS,\cS})^{-1}\bfK_{\cS,\bfW_T}\bfy_T\\
    &\widetilde\sigma_T(w)=\frac{1}{\tau}\left(k(w,w)-\bfk^{\top}_{\cS}(w)\bfK^{-1}_{\cS,\cS}\bfK_{\cS,\cR}\left(\bfK_{\cR,\cS}\bfK^{-1}_{\cS,\cS}\bfK_{\cS,\cR}+\tau\mathbf{Id}\right)^{-1}\bfK_{\cR, \cS}\bfK^{-1}_{\cS,\cS}\bfk_{\cS}(w)\right)
\end{align*}

\end{lemma}

\begin{proof}
    
\begin{align}
\widehat\mu_T(w)&=\phi^{\top}(w)(\cP_{\cS}\Phi_{\cR}\Phi^{\top}_{\cR}\cP_{\cS}+\tau\mathbf{Id})^{-1}\cP_{\cS}\Phi_{\bfW_T}\bfy_T=\\
    &=\frac{1}{\tau}\phi^{\top}(w)\tildeZ^{-1}_{\cS}(\tildeZ_{\cS}-\cP_{\cS}\Phi_{\cR}\Phi^{\top}_{\cR}\cP_{\cS})\Phi_{\cS}\bfK^{-1}_{\cS,\cS}\bfK_{\cS,\bfW_T}\bfy_T=\\
    &=\frac{1}{\tau}\left(\bfk_{\cS}(w)-\phi^{\top}(w)\tildeZ^{- 1}_{\cS}\cP_{\cS}\Phi_{\cR}\Phi^{\top}_{\cR}\cP_{\cS}\Phi_{\cS}\right)\bfK^{-1}_{\cS,\cS}\bfK_{\cS,\bfW_T}\bfy_T=\\
    &=\frac{1}{\tau}\left(\bfk_{\cS}(w)-\phi^{\top}(w)\cP_{\cS}\Phi_{\cR}(\bfK_{\cR,\cS}\bfK^{-1}_{\cS,\cS}\bfK_{\cS,\cR}+\tau\mathbf{I}_{\cR})^{-1}\bfK_{\cR,\cS}\right)\bfK^{-1}_{\cS,\cS}\bfK_{\cS,\bfW_T}\bfy_T=\nonumber\\
    &=\frac{1}{\tau}\bfk_{\cS}(w)\left(\mathbf{I}-\bfK^{-1}_{\cS,\cS}\bfK_{\cS,\cR}(\bfK_{\cR,\cS}\bfK^{-1}_{\cS,\cS}\bfK_{\cS,\cR}+\tau\mathbf{I}_{\cR})^{-1}\bfK_{\cR,\cS}\right)\bfK^{-1}_{\cS,\cS}\bfK_{\cS,\bfW_T}\bfy_T=\nonumber\\
    &=\frac{1}{\tau}\bfk_{\cS}(w)\left(\mathbf{I}_{\cS}-\bfK^{-1}_{\cS,\cS}\bfK_{\cS,\cR}\bfK_{\cR,\cS}(\bfK^{-1}_{\cS,\cS}\bfK_{\cS,\cR}\bfK_{\cR,\cS}+\tau\mathbf{I}_{\cS})^{-1}\right)\bfK^{-1}_{\cS,\cS}\bfK_{\cS,\bfW_T}\bfy_T=\nonumber\\
     &=\frac{1}{\tau}\bfk_{\cS}(w)\left(\mathbf{I}_{\cS}-\bfK^{-1}_{\cS,\cS}\bfK_{\cS,\cR}\bfK_{\cR,\cS}(\bfK^{-1}_{\cS,\cS}\bfK_{\cS,\cR}\bfK_{\cR,\cS}+\tau\mathbf{I}_{\cS})^{-1}\right)\bfK^{-1}_{\cS,\cS}\bfK_{\cS,\bfW_T}\bfy_T=\nonumber\\
     &=\bfk_{\cS}(w)(\bfK^{-1}_{\cS,\cS}\bfK_{\cS,\cR}\bfK_{\cR,\cS}+\tau\mathbf{I}_{\cS})^{-1}\bfK^{-1}_{\cS,\cS}\bfK_{\cS,\bfW_T}\bfy_T=\\
     &=\bfk_{\cS}(w)(\bfK_{\cS,\cR}\bfK_{\cR,\cS}+\tau\bfK_{\cS,\cS})^{-1}\bfK_{\cS,\bfW_T}\bfy_T\\
      \widetilde{\sigma}^2_T(w)&=\phi^{\top}(w)(\cP_{\cS}\Phi_{\cR}\Phi^{\top}_{\cR}\cP_{\cS}+\tau\mathbf{Id})^{-1}\phi(w)=\\
        &=\frac{1}{\tau}(k(w,w)-\phi^{\top}(w)\left(\cP_{\cS}\Phi_{\cR}\Phi^{\top}_{\cR}\cP_{\cS}+\tau\mathbf{Id}\right)^{-1}\cP_{\cS}\Phi_{\cR}\Phi^{\top}_{\cR}\cP_{\cS}\phi(w))=\\
        &=\frac{1}{\tau}\left(k(w,w)-\phi^{\top}(w)\cP_{\cS}\Phi_{\cR}\left(\Phi^{\top}_{\cR}\cP_{\cS}\Phi_{\cR}+\tau\mathbf{Id}\right)^{-1}\Phi^{\top}_{\cR}\cP_{\cS}\phi(w)\right)=\\
        &=\frac{1}{\tau}\left(k(w,w)-\bfk^{\top}_{\cS}(w)\bfK^{-1}_{\cS,\cS}\bfK_{\cS,\cR}\left(\bfK_{\cR,\cS}\bfK^{-1}_{\cS,\cS}\bfK_{\cS,\cR}+\tau\mathbf{Id}\right)^{-1}\bfK_{\cR, \cS}\bfK^{-1}_{\cS,\cS}\bfk_{\cS}(w)\right)\nonumber
\end{align}
In the above derivation we used Lemma \ref{kernel_trick_mat_identity} along with the parametric form of the projection operator $\cP_{\cS}=\Phi_{\cS}\bfK^{-1}_{\cS,\cS}\Phi^{\top}_{\cS}$ 
\end{proof}

\newpage

\section{Appendix C. \textsc{CAPRI} algorithm}\label{App:algortihm}

In this section we present the pseudo-code of \textsc{CAPRI} algorithm.\\

\begin{algorithm}[H]
\caption{\textsc{CAPRI}}
\begin{algorithmic}[1]\label{Algortihm:Code}
    \STATE \textbf{Input}: Error probability $\delta_{\textsc{err}}$, Privacy parameters $(\varepsilon_{\textsc{DP}}, \delta_{\DP})$, 
    \STATE $r\leftarrow 1, g \leftarrow 0, t_{\text{prev}} \leftarrow 0, \cS, \cR \leftarrow \emptyset, T_1 \leftarrow\lceil\sqrt{T} \rceil $
    \STATE $\cX_1(c) = \cX \ \forall c \in \cC$ 
    \FOR{$s = 1,2,\dots, T_1$}
    \STATE Sample $c, c' \sim \kappa$
    \STATE Sample $x \sim \mathrm{Unif}(\cX_1(c)), x' \sim \mathrm{Unif}(\cX_1(c'))$
    \STATE $\cS \leftarrow \cS \cup \{(c,x)\}, \cR \leftarrow \cR \cup \{(c',x')\}$
    \ENDFOR
    \STATE Compute $\sigma_{\max}$ using Eqn.~\eqref{eqn:projected_variance}
    \FOR{$t = 1,2,\dots, T$}
    \STATE Receive context $c_t$ 
    \STATE Take the action $x_t \sim \mathrm{Unif}(\cX_r(c_t))$ and observe reward $y_t$
    \STATE $g^{(t)} \leftarrow y_t(\bfK_{\cS,\cR}\bfK_{\cR,\cS}+\tau\bfK_{\cS,\cS})^{-1/2}k_{\cS}((c_t,x_t))$
    \STATE \blue{$g^{(t)} \leftarrow g^{(t)} + \cN(0, \sigma_0^2 \cdot \bfI_{T_r})$} \hfill \blue{\texttt{// Only for LDP}}
    \STATE $g \leftarrow g + g^{(t)}$
    \IF{$t - t_{\text{prev}} = T_r$}
    \STATE \red{$g \leftarrow g + \cN(0, \sigma_0^2\cdot\bfI_{T_r})$} \hfill \red{\texttt{// Only for JDP}}
    \STATE $\widehat{\mu}_r(\cdot) := k^{\top}_{\cS}(\cdot)(\bfK_{\cS,\cR}\bfK_{\cR,\cS}+\tau\bfK_{\cS,\cS})^{-1/2}g$
    \STATE Obtain $\cX_{r+1}(c)$ for all $c \in \cC$ using Eqn.~\eqref{eqn:prunning_sets}
   	\STATE $\cS, \cR \leftarrow \emptyset, g \leftarrow 0, t_{\text{prev}} \leftarrow t, T_{r+1} \leftarrow 2T_r, r \leftarrow r+1$
   	\FOR{$s = 1,2,\dots, T_r$}
    \STATE Sample $c, c' \sim \kappa$
    \STATE Sample $x \sim \mathrm{Unif}(\cX_r(c)), x' \sim \mathrm{Unif}(\cX_r(c'))$
    \STATE $\cS \leftarrow \cS \cup \{(c,x)\}, \cR \leftarrow \cR \cup \{(c',x')\}$
    \ENDFOR
    \STATE Compute $\widetilde\sigma_{\max}$ using Eqn.~\eqref{eqn:projected_variance}
    \ENDIF
    \ENDFOR
\end{algorithmic}
\end{algorithm}

The \textsc{CAPRI} algorithm is based on the \emph{explore-and-eliminate} philosophy. It operates in epochs that progressively double in length. For each epoch $r$, \textsc{CAPRI} maintains the set of active actions, $\cX_r(c) \subseteq \cX$, corresponding to each context $c \in \cC$. These sets are initialized to $\cX_1(c) = \cX$ for all $c \in \cC$. At the beginning of each epoch $r$, the algorithm has access to two sets $\cR_r$ and $\cS_r$ consisting of $T_r$ context-action pairs, where $T_r$ is the length of the $r^{\text{th}}$ epoch. During the $r^{\text{th}}$ epoch and time instant $t$, the algorithm observes the context $c_t$ and samples $x_t$ uniformly from the set $\cX_r(c_t)$. It is straightforward to note that the points generated during each epoch are independent and identically distributed. It uses $\cR_r$ and $\cS_r$ to construct $\widehat{\underline{\mu}}_r$ as outlined in Eqn.~\eqref{eqn:private_estimator} and update the action sets for the $r+1$-th epoch:
\begin{align}
	&\cX_{r+1}(c)= \left\{x\in \cX_{r}(c) \ \bigg| \  \widehat{\underline{\mu}}_r(c,x)\geq \max_{x\in \cX_{r}(c)}\widehat{\underline{\mu}}_r(c,x)-4\Delta_r\right\} \nonumber \\
&\Delta_r :=\beta\left(\frac{\delta}{|\cW|T\log{T}}\right)\widetilde\sigma_{r,\max}+\beta_1\left(\varepsilon_{\textsc{DP}},\delta_{\textsc{DP}},\frac{\delta}{|\cW|T\log{T}}\right)\widetilde\sigma^2_{r,\max}
\end{align}
where $\beta_1=\frac{\log{T}\sqrt{8B\log(\log{T}|\cW|/\delta)\log(1.25\log{T}/\delta_{\textsc{DP}})}}{\varepsilon_{\textsc{DP}}}$, $\beta_{1,\textsc{LDP}}=\sqrt{T_r}\beta_1$ and $\beta$ is the confidence parameter(for a full expression please see Lemma~\ref{Lemma:Appendix_estimator}).
At the end of the $r^{\text{th}}$ epoch, \textsc{CAPRI} constructs the collections $\cR_{r+1}$ and $\cS_{r+1}$, to be used in the next epoch. Both sets consist of i.i.d. $T_{r+1}$ context-action pairs obtained by first drawing $c \sim \kappa$ from the context generator and then drawing the corresponding action $x$ uniformly from the set $\cX_{r+1}(c)$. For more details on the algortihm please see Sec.\ref{App:algortihm}.

The algorithm is identical under both JDP and LDP, except for the construction of the estimator which is used to update the action sets $\cX_r(c)$. We  point out that as shown in pseudocode outlined in Alg.\ref{Algortihm:Code}, the construction of the estimator under JDP and LDP is carried out differently to satisfy the privacy constraint. In addition to different to noise addition schedules, confidence parameter $\beta_1$ is designed to reflect  greater uncertainty in $\mu_{1, \text{LDP}}$. Namely we set $\beta_{1,\textsc{LDP}}=\sqrt{T_r}\beta_1$.

\newpage

\newpage

\end{document}